\documentclass{article}
\usepackage{lmodern}
\usepackage{amsmath}
\usepackage{amsthm}
\usepackage{graphicx}
\usepackage{algorithm}
\usepackage{algorithmic}
\usepackage{wrapfig}
\newtheorem{remark}{Remark}
\newtheorem{theorem}{Theorem}
\newtheorem{corollary}{Corollary}

\newtheorem{lemma}{Lemma}
\usepackage[preprint]{neurips_2026}

\usepackage[utf8]{inputenc} % allow utf-8 input
\usepackage[T1]{fontenc}    % use 8-bit T1 fonts
\usepackage{url}            % simple URL typesetting
\usepackage{booktabs}       % professional-quality tables
\usepackage{amsfonts}       % blackboard math symbols
\usepackage{nicefrac}       % compact symbols for 1/2, etc.
\usepackage{microtype}      % microtypography
\usepackage{xcolor}         % colors

\usepackage{xcolor}
\definecolor{mydarkred}{rgb}{0.6,0,0}
\definecolor{mydarkgreen}{rgb}{0,0.6,0}
\usepackage[colorlinks,
linkcolor=mydarkred,
citecolor=mydarkgreen]{hyperref}

\title{A Regret Perspective on Online Multiple Testing}

\author{%
  Qingyang Hao \quad Kongchang Zhou \quad Fang Kong \quad Hongxin Wei \\ 
  Southern University of Science and Technology 
}

\begin{document}

\maketitle

\begin{abstract}
% Online Multiple Testing (OMT) is a fundamental pillar of sequential statistical inference, traditionally evaluates the False Discovery Rate (FDR) and statistical power in isolation, obscuring the highly asymmetric costs of false positives and false negatives in modern automated pipelines. 
Online Multiple Testing (OMT), a fundamental pillar of sequential statistical inference, traditionally evaluates the False Discovery Rate (FDR) and statistical power in isolation, obscuring the highly asymmetric costs of false positives and false negatives in modern automated pipelines. 
To unify this evaluation, we introduce \textit{Weighted Regret}. Under this metric, we prove the \textit{Duality of Regret Conservation}: purely deterministic procedures ensuring strict FDR control inevitably incur an $\Omega(T)$ linear regret penalty, as threshold depletion during signal-sparse cold starts forces massive false negatives. 
Tailored for exogenous testing streams, we propose Decoupled-OMT (DOMT) as a baseline-agnostic meta-wrapper. 
% By incorporating a history-decoupled, strictly non-negative random perturbation, DOMT rescues 
% % both classical $p$-value and modern $e$-value 
% baselines from zero-threshold deadlocks. 
By incorporating a history-decoupled, strictly non-negative random perturbation, DOMT rescues purely deterministic baselines from severe threshold depletion. 
% Crucially, it preserves their exact asymptotic safety in stationary environments and rigorously bounds finite-sample error inflation during extreme cold-starts. It guarantees zero additional false negatives and yields an order-optimal $\Omega(\sqrt{T})$ regret reduction in adversarial bursty environments. We further derive a closed-form "Cold-Start Tax" characterizing the phase transition of algorithmic superiority. 
Crucially, it preserves exact asymptotic safety in stationary environments and rigorously bounds finite-sample error inflation during cold-starts. Guaranteeing zero additional false negatives, it yields an order-optimal $\Omega(\sqrt{T})$ regret reduction in bursty environments, with a derived ``Cold-Start Tax'' characterizing the exact phase transition of algorithmic superiority.
Experiments validate that DOMT consistently curtails empirical weighted regret, achieving an order-optimal sublinear mitigation of threshold depletion to navigate the non-stationary Pareto frontier.

\end{abstract}

\vspace{-1mm}

\section{Introduction}

% Multiple hypothesis testing is a core problem in statistical inference and arises in almost every scientific field.
% Instead of testing $N$ hypotheses at once, Online Multiple Testing (OMT) focuses on the online setting, where a stream of hypotheses arrives online.
% At each step, the analyst must decide whether to reject the current null hypothesis solely based on the previous decisions and evidence against the current hypothesis.
% However, the prevailing OHT paradigm defaults to controlling the FDR, i.e., the expected proportion of false discoveries.
% Consequently, these algorithms tend to be overly conservative.
% False Nondiscovery Rate

\vspace{-1mm}

Multiple hypothesis testing \cite{shaffer1995multiple, efron2012large} is a core problem in statistical inference and arises in almost every scientific field. Instead of testing a fixed batch of hypotheses at once, Online Multiple Testing (OMT) \cite{foster2008alpha, javanmard2018online} focuses on the sequential setting, where a continuous stream of hypotheses arrives exogenously over time. 
At each step, the algorithm must make an immediate and irrevocable decision to either reject or retain the current null hypothesis, relying solely on previous decisions and current evidence. 
Fundamentally, this process involves an inherent trade-off between Type I errors (false discoveries) and Type II errors (missed discoveries).
Existing OMT works \cite{foster2008alpha, Yang2017AFF, ramdas2018saffron, tian2019addis, lin2024online, Chen2019ContextualOF, Xu2023OnlineMT, Kuang2026SCOREAU} formally cast this as the challenge of maintaining the False Discovery Rate (FDR) below a predefined target level $\alpha$. 
To guarantee this error control over an infinite horizon, algorithms typically manage a finite testing capacity (often termed $\alpha$-wealth) while striving to achieve high statistical power \cite{benjamini1995controlling, javanmard2018online}. 
Given the high stakes of modern automated pipelines, designing robust strategies to effectively navigate this trade-off remains a critical open challenge.

% \vspace{-1mm}

However, previous OMT works primarily focus on strict FDR control without explicit constraints on statistical power. 
Such a paradigm often forces algorithms to be overly conservative, diminishing their practical value in real-world applications. 
For example, periods with sparse true signals may deplete the testing capacity ($\alpha$-wealth) \cite{javanmard2018online}, resulting in monotonic decay of history-dependent thresholds toward zero \cite{NIPS2017_7f018eb7}.
% This indicates the flaw of the objective, as well as the evaluation metrics.
% consequence
% this motivates us to design a new OMT objective that captures the asymmetric trade-off between ....
This vulnerability indicates a flaw in the traditional optimization objective, as well as its underlying evaluation metrics. 
Consequently, by evaluating FDR and power in isolation, conventional paradigms struggle to recover from severe threshold depletion and obscure the highly asymmetric consequences of false positives and false negatives in high-stakes AI domains \cite{johari2022experimental, Angelopoulos2022ConformalRC, DBLP:conf/iclr/PodkopaevR22}. 
This critical gap motivates us to design a novel OMT objective that explicitly captures and dynamically balances the asymmetric trade-off between Type I and Type II errors over time.

To fulfill this objective, we evaluate decisions against an omniscient Oracle. 
Building upon classical regret \cite{Savage1951TheTO, Lai1985AsymptoticallyEA}, we introduce \textit{Weighted Regret} to directly minimize this asymmetric trade-off, effectively resolving the fundamental incompatibility of traditional metrics. 
Instead of struggling to directly compare a conditional rate (FDR) and a marginal probability (power), weighted regret maps both Type I and Type II errors into a single, cohesive currency of cumulative loss. 
Under this unified framework, we establish the \textit{Duality of Regret Conservation}, mathematically proving that purely deterministic algorithms, whether maintaining constant power or enforcing strict FDR control, inevitably incur an $\Omega(T)$ linear regret penalty. 
This fundamental limit dictates that purely history-dependent frameworks are mathematically trapped in a regime of linear loss, necessitating a departure from passively bounding errors to actively countering deterministic threshold depletion.

Driven by this philosophy, we propose the \textit{Decoupled-OMT} (DOMT) algorithm, which mitigates classical limitations by incorporating a history-decoupled, strictly non-negative random perturbation into sequential decisions. 
This unidirectional stochastic exploration ensures the rejection set of DOMT always subsumes that of the deterministic baseline. 
By guaranteeing zero additional false negatives on any sample path, DOMT translates this exploration into a provable enhancement in detection power. 
Crucially, it retains exact asymptotic FDR control in stationary environments, while bounding finite-sample error inflation during extreme adversarial cold-starts. 
Under bursty signal scenarios, we prove that DOMT achieves an order-optimal $\Omega(\sqrt{T})$ reduction in \textit{Weighted Regret}. 
Finally, we derive a closed-form phase transition threshold, conceptualized as the ``Cold-Start Tax,'' dictating exactly when this exploration becomes theoretically advantageous, achieving an order-optimal sublinear mitigation of threshold depletion to navigate the non-stationary Pareto frontier.

In summary, our main \textbf{contributions} are:
\textbf{(i) Metric}: Introducing \textit{Weighted Regret} and proving the \textit{Duality of Regret Conservation}, establishing the inevitable $\Omega(T)$ linear penalty of purely deterministic algorithms (Sections \ref{sec:impossibility}). 
\textbf{(ii) Algorithm}: Proposing DOMT via strictly non-negative perturbations to mitigate threshold depletion. It provably enhances detection power while retaining stationary FDR control and bounding cold-start error inflation (Section \ref{sec:algorithm}).
\textbf{(iii) Analysis}: Proving an order-optimal $\Omega(\sqrt{T})$ regret reduction and deriving a closed-form ``Cold-Start Tax'' that characterizes the phase transition of algorithmic advantage in bursty environments (Section \ref{sec:regret_reduction}). 
\textbf{(iv) Validation}: Empirically corroborating the theoretical threshold dynamics and Pareto frontier traversal, demonstrating robust performance beyond strict analytical assumptions (Section \ref{sec:experiments}).
A roadmap of our supplementary proofs, extended discussions, and additional experiments is provided in \hyperref[apporg]{Appendix}.

\section{Related work}
\label{rlw}

\vspace{-1mm}
To navigate the sequential decision-making process in OMT, extensive literature has focused on designing rigorous error-control mechanisms. Specifically,
to guarantee $\text{FDR}_T \le \alpha$, classical procedures rely on summable threshold sequences ($\sum \gamma_t \le 1$) \cite{foster2008alpha, javanmard2018online}, comprehensively reviewed in \cite{Robertson2022OnlineMH}. These were later enhanced by adaptive discarding \cite{ramdas2018saffron, tian2019addis} and 
highly flexible $e$-value martingales \cite{Kuang2026SCOREAU, Zhang2025eGAIEG}. Despite elegant overshoot-refund mechanisms and recent stochastic extensions sharing a similar exploration spirit (e.g., U$e$-LOND \cite{xu2024more, fischer2024online, Xu2023OnlineMT}), directly coupling exploration noise into their wealth states inevitably leaves them vulnerable to irreversible threshold depletion 
($\lambda_t \to 0$) during macroscopic pure-null droughts or asynchronous bursty shifts \cite{zrnic2021asynchronous}. Concurrently, to reconcile asymmetric costs, recent literature integrates FDR with regret minimization. However, these frameworks typically necessitate \textit{active} sequential interventions (e.g., MAB/selective generation \cite{Yang2017AFF, Chiong2023TheMD, leebandit}) or rely on Local FDR posterior oracles for resource allocation \cite{lin2024online}. 
Distinct from these active paradigms, our work rigorously focuses on assumption-free passive Global FDR testing streams.
By unifying evaluation under \textit{Weighted Regret}, DOMT acts as a cross-paradigm meta-wrapper. It uniquely employs causal-decoupled stochastic exploration to rescue both classical $p$-value and modern $e$-value baselines from zero-threshold deadlocks without compromising their native super-martingale safety. 
(More discussion on related work is deferred to Appendix \ref{F-rlw}).

\section{Problem setup}
\label{sec:setup}

\vspace{-1mm}
In this section, we formalize the protocol of OMT, outline the specific mathematical assumptions required for our subsequent theoretical analysis, and review the traditional evaluation metrics.

% \vspace{-1mm}
\textbf{The sequential decision protocol.} Consider a sequential process over $T$ rounds. At round $t$, the environment generates a true state $Y_t \in \{0, 1\}$ with a global null proportion $\pi \in (0,1)$, and reveals a $p$-value $p_t \in [0,1]$. Specifically, $p_t \sim U[0,1]$ if $Y_t = 0$ (null), and $p_t \sim G$ if $Y_t = 1$ (alternative). Given the historical filtration $\mathcal{F}_{t-1} = \sigma((p_s, \lambda_s, \delta_s) : s \le t-1)$, the algorithm dynamically selects a threshold $\lambda_t \in [0,1]$ to yield a binary rejection decision $\delta_t = \mathbf{1}\{p_t \le \lambda_t\}$.

\begin{remark}
Note that the specific temporal structure of the state sequence $\{Y_t\}_{t=1}^T$ (e.g., stationary processes for general lower bounds, or piece-wise deterministic sequences for bursty adversarial scenarios) is not globally fixed here. It will be strictly formalized in the respective theoretical analyses in subsequent sections.
\end{remark}

\vspace{-2mm}
\textbf{Assumptions on the alternative distribution.} Let $G: [0,1] \rightarrow [0,1]$ denote the cumulative distribution function (CDF) of the alternative $p$-values, which inherently dictates the detection power since $\mathbb{P}(p_t \le \lambda_t \mid Y_t=1) = G(\lambda_t)$. We assume $G$ is continuous, strictly increasing, and satisfies $G(0)=0$ and $G(1)=1$. To ensure rigorous regret quantification, we impose two key conditions. First, $G$ satisfies a global Lipschitz condition: there exists $L \ge 1$ such that $|G(x) - G(y)| \le L|x-y|$ for all $x, y \in [0,1]$, which uniformly bounds the alternative $p$-value density. Second, to guarantee discovery recovery, the alternative signal must maintain a minimum statistical detectability: there 
exists $\mu > 0$ and a neighborhood $x_0 > 0$ such that $G(x) \ge \mu x$ for all $x \in [0, x_0]$.

\begin{remark}[Weak signals and conservative bounds]
The Lipschitz condition, implying $G(x) \le Lx$, characterizes the most challenging ``weak signal'' environments in statistical testing. Consequently, the $\Omega(\sqrt{T})$ regret reduction derived subsequently serves as a conservative, worst-case lower bound. In stronger signal regimes (e.g., where the $p$-value density diverges near zero), DOMT's discovery recovery accelerates, substantially exceeding the $\Omega(\sqrt{T})$ magnitude. As demonstrated in Section \ref{sec:experiments}, DOMT maintains robust performance even when these strict analytical constraints are relaxed.
\end{remark}

\vspace{-1mm}

\textbf{Evaluation metrics.} Traditionally, OMT procedure is evaluated by balancing two metrics: $\text{FDR}_T = \mathbb{E}[V_T / \max(1, R_T)]$ for error control, and $\text{Power}_T = \mathbb{E}[S_T / \max(1, \sum Y_t)]$ for signal detection.

\section{The inevitability of linear regret}
\label{sec:impossibility}

% \vspace{-2mm}
In this section, we evaluate the performance of deterministic OMT paradigms within the framework of the proposed \textit{Weighted Regret}. 
We prove that the inherent constraint of either maintaining a minimum detection power or strictly bounding the FDR inextricably leads to linear \textit{Weighted Regret}.

\subsection{The Weighted Regret metric}

% \textbf{The Weighted Regret metric.} 
To unify the evaluation of FDR and detection Power, we quantify the cumulative asymmetric costs of decisions. Let $V_T = \sum_{t=1}^T \mathbf{1}\{Y_t=0, \delta_t=1\}$ and $M_T = \sum_{t=1}^T \mathbf{1}\{Y_t=1, \delta_t=0\}$ denote the total false positives (type I errors) and false negatives (type II errors), respectively. Given domain-specific penalty weights $a, b > 0$, the \textit{Weighted Regret} is formulated as:
\begin{equation}
    \text{Regret}_T(a,b) = a \cdot V_T + b \cdot M_T
\end{equation}
Since an omniscient Oracle incurs zero errors, the algorithm's expected excess risk mathematically equals this total cumulative cost. Consequently, the objective shifts to minimizing the expected regret $\mathbb{E}[\text{Regret}_T(a,b)]$ in non-stationary environments, while simultaneously satisfying a foundational statistical constraint (e.g., asymptotic FDR control or a minimum power).

\begin{remark}[Degenerate regimes and convex parameterization]
We strictly assume $a, b > 0$, as setting $a=0$ or $b=0$ trivially reduces the optimal strategy to rejecting ($\lambda_t \equiv 1$) or accepting ($\lambda_t \equiv 0$) all hypotheses, respectively. While we retain the unnormalized penalties $(a,b)$ for algebraic clarity, the objective equivalently admits a convex parameterization $w \cdot V_T + (1-w) \cdot M_T$, where $w = a/(a+b) \in (0,1)$. This maps the asymmetric ratio to $(1-w)/w$, elegantly decoupling subjective risk aversion from the overall scale of environmental costs.
\end{remark}

% \vspace{-2mm}
\subsection{The power barrier: the cost of constant vigilance}

To avoid total blindness, suppose an algorithm maintains a minimum detection power: $\mathbb{P}(\delta_t = 1 \mid Y_t = 1, \mathcal{F}_{t-1}) = G(\lambda_t) \ge \beta > 0$ a.s. for all $t$. This necessitates a strictly positive threshold lower bound: $\lambda_t \ge \lambda_\beta := G^{-1}(\beta) > 0$. Theorem \ref{thm:power_lower_bound} proves that such a constraint inevitably triggers an $\Omega(T)$ linear weighted regret, holding both in expectation and almost surely (a.s.). Crucially, this persistent false-positive penalty manifests even in a stationary environment where $Y_t \sim \text{i.i.d. Bernoulli}(1-\pi)$.
The proofs are deferred to Appendix \ref{A-1}.

\begin{theorem}[Weighted linear regret lower bound]
\label{thm:power_lower_bound}
Assume a stationary environment where true states $Y_t \sim \text{i.i.d. Bernoulli}(1-\pi)$ for a constant $\pi \in (0,1)$. If an algorithm maintains a minimum detection power $\beta > 0$ (i.e., $\lambda_t \ge \lambda_\beta > 0$ a.s. $\forall t$), the Weighted Regret is inherently linear. Specifically, for any $T \ge 1$, the expected regret is $\Omega(T)$, and as $T \rightarrow \infty$, this linear penalty holds almost surely (a.s.):
\begin{equation}
    \mathbb{E}[\text{Regret}_T(a,b)] \ge a \pi \lambda_\beta T = \Omega(T), \quad \text{and} \quad \liminf_{T \rightarrow \infty} \frac{1}{T} \text{Regret}_T(a,b) \ge a \pi \lambda_\beta \quad \text{a.s.}
\end{equation}
\end{theorem}

\begin{remark}[The tax of constant vigilance]
The lower bound $a \cdot \pi \lambda_\beta$ reveals a fundamental ``linear tax'' inherent to constant power. To maintain sensitivity $\beta$, the algorithm must keep its testing ``net'' permanently open at a size $\lambda_\beta$. Because null $p$-values are uniform, a fraction $\lambda_\beta$ of the nulls ($\pi$) will inevitably trigger false alarms by random chance. This irreducible accumulation of type I errors forces the weighted regret to grow linearly, strictly dominated by the false-positive penalty $a$.
\end{remark}

\subsection{The FDR barrier: the trap of threshold depletion}

To secure strict FDR control, classical deterministic algorithms enforce a monotonic decay on their worst-case threshold sequences, satisfying $\sum_{t=1}^{\infty} \tilde{\lambda}_t \le C < 1$. While this conservative mechanism tightly bounds the expected false-positive regret to $O(1)$, Theorem \ref{thm:fdr_barrier} proves that it inevitably traps the algorithm in a state of global blindness upon encountering non-stationary, bursty signals.

\begin{theorem}[The inevitable false negative penalty of threshold decay]
\label{thm:fdr_barrier}
Consider an algorithm whose worst-case threshold sequence satisfies $\sum_{t=1}^{\infty} \tilde{\lambda}_t \le C < 1$. Assume $G(x) \le Lx$ for some constant $L \ge 1$. For any even integer $T \ge 2$, let the environment generate a two-phase deterministic sequence: $Y_t = 0$ for $1 \le t \le T/2$, and $Y_t = 1$ for $T/2 < t \le T$. 
For sufficiently large $T$, the expected Weighted Regret satisfies:
\begin{equation}
\small
    \mathbb{E}[\text{Regret}_T(a,b)] \ge b \cdot \frac{T}{2} \left[ (1-C) \prod_{t=T/2+1}^{\infty} (1 - L\tilde{\lambda}_t)_+ \right] = \Omega(T).
\end{equation}
\end{theorem}

The complete proof of Theorem \ref{thm:fdr_barrier} is deferred to Appendix \ref{A-2}. 
To concretely instantiate Theorem \ref{thm:fdr_barrier}, we evaluate three canonical algorithms: LOND ($\tilde{\lambda}_t \le \alpha \gamma_t$), LORD ($\tilde{\lambda}_t \le W_0 \gamma_t$), and SAFFRON ($\tilde{\lambda}_t \le (1-\lambda) W_0 \gamma_t$) \cite{javanmard2018online, ramdas2018saffron}. Under the worst-case null phase, their deterministic threshold decay ($\sum \gamma_t \le 1$) inherently bounds the expected false positives to $\mathcal{O}(1)$. However, the Lipschitz condition forces their detection power to simultaneously vanish, 
yielding a linear false-negative penalty of $\Omega(T)$ for all three methods (specifically scaling as $T/2 - \mathcal{O}(1)$ under standard weak signals). 
Consequently, their overall \textit{Weighted Regret} is dominated by these missed discoveries, scaling as $b \cdot \mathbb{E}[M_T] = \Omega(T)$. 
The complete derivations are deferred to Appendix \ref{A-3}.

\subsection{The Duality of Regret Conservation}

Theorems \ref{thm:power_lower_bound} and \ref{thm:fdr_barrier} establish a fundamental \textit{Duality of Regret Conservation}. Within classical deterministic OMT frameworks, avoiding one error barrier inevitably triggers the other, dictating a universal linear lower bound: $\mathbb{E}[\text{Regret}_T(a,b)] = \Omega(T)$. Specifically, the algorithm incurs either an $\Omega(aT)$ false-positive penalty by maintaining minimum power, or an $\Omega(bT)$ false-negative penalty via threshold depletion. As formalized in Corollary \ref{cor:impossibility}, this dichotomy precludes any purely history-dependent mechanism from simultaneously achieving sublinear regrets for both errors.

\begin{corollary}[Impossibility of simultaneous sublinear regrets]
\label{cor:impossibility}
For any deterministic, purely history-dependent online testing algorithm, there exists an environment sequence where it cannot simultaneously achieve $\mathbb{E}[V_T] = o(T)$ and $\mathbb{E}[M_T] = o(T)$. Consequently, its worst-case expected weighted regret is unconditionally linear, $\Omega(T)$ (proof in Appendix \ref{A-4}).
\end{corollary}

\begin{remark}[The necessity of stochastic exploration]
This strict impossibility confirms that linear regret stems not from suboptimal parameter tuning, but from the fundamental structural limits of deterministic mechanisms. Mitigating the severity of this conservation law necessitates a paradigm shift toward history-decoupled stochastic exploration, setting the exact stage for our DOMT framework.
\end{remark}

\section{The Decoupled-OMT algorithm and safety guarantees}
\label{sec:algorithm}

% \vspace{-2mm}
To circumvent the structural limitations of deterministic mechanisms identified in Section \ref{sec:impossibility}, we propose Decoupled-OMT (DOMT) as a versatile framework equipping classical online procedures with stochastic exploration. We first formalize the DOMT architecture and establish its rigorous safety, proving exact asymptotic FDR inheritance and a finite-sample bound for graceful degradation. Finally, our general regret decomposition proves the worst-case exploration overhead is strictly bounded by $\mathcal{O}(\sqrt{T})$, fundamentally preserving the underlying optimal trajectory.

\subsection{The Decoupled-OMT mechanism}

The fundamental innovation of DOMT is a targeted threshold randomization strategy designed as a versatile, plug-and-play module. Rather than being limited to a specific procedure, DOMT is widely applicable to a broad family of classical online algorithms. To prevent procedures from being trapped in states of severe threshold depletion, DOMT introduces a decaying, history-decoupled random perturbation to actively explore the $p$-value space. Specifically, at each round $t$, the decision threshold $\lambda_t$ is constructed as:
\begin{equation}
    \lambda_t = \min(1, \lambda_t^{\text{base}} + \xi_t).
\end{equation}
Here, $\lambda_t^{\text{base}}$ serves as the \textit{virtual baseline threshold} dictated by the chosen underlying procedure (for instance, if LOND \cite{javanmard2018online} is employed, $\lambda_t^{\text{base}} = \alpha \cdot \gamma_t \cdot \max(1, R_{t-1}^{\text{base}})$, where $R_{t-1}^{\text{base}}$ is the virtual rejection count). The perturbation term is defined as $\xi_t = \epsilon_t Z_t$, where $Z_t \sim \text{Uniform}[0, 1]$ is i.i.d., 
and $\epsilon_t = \kappa \cdot \frac{\alpha}{\sqrt{t}}$
represents the exploration amplitude with user-defined coefficient $\kappa \ge 0$.

Crucially, DOMT guarantees an unconditional improvement in detection power without compromising asymptotic FDR control via a targeted decoupling mechanism. Because the perturbations are strictly non-negative, the actual rejection set deterministically subsumes the baseline's, ensuring zero additional false negatives ($M_T^{\text{DOMT}} \le M_T^{\text{base}}$) on any sample path. To prevent this active exploration from polluting future baseline calculations, the internal state of the chosen procedure is updated solely using the virtual thresholds $\lambda_t^{\text{base}}$, while actual decisions rely on the perturbed $\lambda_t$. This effectively isolates the noise, preserving the underlying trajectory required for theoretical safety. 
The complete procedure is formalized in Algorithm \ref{alg:dbh}.

\begin{algorithm}[H]
\caption{The Decoupled-OMT (DOMT) Algorithm}
\label{alg:dbh}
\begin{algorithmic}[1]
% \small
\REQUIRE Target FDR level $\alpha \in (0,1)$, a monotone base procedure, exploration coefficient $\kappa \ge 0$.
\ENSURE Rejection set $\mathcal{R}_T$.
\STATE Initialize: $\mathcal{R}_0 \leftarrow \emptyset$, and initialize the internal state of the base procedure.
\FOR{$t = 1, 2, \dots, T$}
    \STATE \textbf{1. Baseline Calculation:} Compute $\lambda_t^{\text{base}}$ according to the rules of the chosen base procedure.
    \STATE \textbf{2. Stochastic Exploration:} Sample $Z_t$ and compute $\epsilon_t \leftarrow \kappa \cdot \frac{\alpha}{\sqrt{t}}$. Set $\xi_t \leftarrow \epsilon_t \cdot Z_t$.
    \STATE \textbf{3. Threshold Clipping:} Set $\lambda_t \leftarrow \min(1, \lambda_t^{\text{base}} + \xi_t)$.
    \STATE \textbf{4. Virtual Update:} Update base state via virtual decision $\delta_t^{\text{base}} = \mathbf{1}\{p_t \le \lambda_t^{\text{base}}\}$.
    \STATE \textbf{5. Actual Decision:} \textbf{if} $p_t \le \lambda_t$ \textbf{then} $\mathcal{R}_t \leftarrow \mathcal{R}_{t-1} \cup \{t\}$ \textbf{else} $\mathcal{R}_t \leftarrow \mathcal{R}_{t-1}$.
\ENDFOR
\RETURN $\mathcal{R}_T$.
\end{algorithmic}
\end{algorithm}

\begin{remark}[Exploration coefficient and weight-adaptivity]
\label{rem:weight_adaptive}
DOMT accommodates any $\kappa > 0$, preserving exact FDR safety and $\mathcal{O}(\sqrt{T})$ regret bounds up to a constant factor (Theorems \ref{thm:fdr_control}, \ref{thm:finite_fdp}, \ref{thm:regret_gap}), while enabling flexible traversal of the $(V_T, M_T)$ Pareto frontier. If the penalty ratio $M = b/a$ is known, adopting an increasing mapping $\kappa(M)$ heuristically mitigates false negatives when $b \gg a$. 
% However, deriving the optimal $\kappa^*(M)$ remains an open problem due to its dependence on the unknown signal distribution.
However, deriving the strictly optimal mapping $\kappa^*(M)$ remains a pivotal open problem, precisely because it intrinsically depends on the unknown underlying signal distribution.
\end{remark}

\begin{remark}[Immunity to spurious cascades]
A natural theoretical concern is whether the stochastic perturbation $\xi_t$ could trigger a spurious ``rich-get-richer'' cycle, where noise-induced false alarms uncontrollably inflate future thresholds. DOMT intrinsically eradicates this vulnerability through its \textit{causal decoupling}. As defined in Algorithm \ref{alg:dbh}, the virtual rejection count $R_t^{\text{base}}$ is updated exclusively via the unperturbed baseline $\lambda_t^{\text{base}}$. Consequently, the exploration noise functions as a strictly ``forward-only'' mechanism: any additional discoveries induced by $\xi_t$ are mathematically quarantined within the actual rejection set $\mathcal{R}_t$ and never leak back to corrupt future baselines, inherently preserving the super-martingale safety of modern $e$-value procedures (Appendix \ref{D-2}).
\end{remark}

\subsection{Asymptotic FDR control of DOMT}

We first establish that the strictly non-negative stochastic perturbations, scaled by any user-defined exploration coefficient $\kappa > 0$, do not compromise asymptotic safety. Theorem \ref{thm:fdr_control} formally guarantees that the DOMT framework strictly inherits the FDR control of its chosen base procedure. For readability, standard regularity conditions (including the Lipschitz continuity of null $p$-values and non-vanishing discovery rates) and the complete derivation are deferred to Appendix \ref{B}.

\begin{theorem}[Asymptotic FDR control of DOMT]
\label{thm:fdr_control}
Subject to the regularity conditions specified in Appendix \ref{B-1}, if the chosen base procedure satisfies $\limsup_{t\rightarrow\infty} \text{FDR}_t^{\text{base}} \le \alpha$, then the DOMT framework (Algorithm \ref{alg:dbh}) maintains its safety profile: $\limsup_{t\rightarrow\infty} \text{FDR}_t^{\text{DOMT}} \le \alpha$.
\end{theorem}

\begin{remark}[Algorithmic trade-off in adversarial regimes]
\label{rem:phase_transition}
While the perturbation inevitably incurs a sublinear false positive tax, macroscopic discoveries rapidly dilute it under regular conditions, preserving asymptotic FDR. Conversely, in extreme adversarial cold-starts (Section \ref{sec:experiments}), the structural absence of true signals dictates that any exploratory rejection will transiently inflate the empirical FDP to $100\%$. Rather than a structural flaw, this represents a deliberate phase transition in risk management. When absolute threshold depletion drives classical algorithms into mathematical blindness, DOMT accepts this localized error penalty to secure an order-optimal $\Omega(\sqrt{T})$ signal mass along the non-stationary Pareto frontier. To rigorously bound the absolute magnitude of this error inflation during such worst-case extremes, we establish Theorem \ref{thm:finite_fdp}.
\end{remark}

\begin{theorem}[Finite-sample FDP inflation bound]
\label{thm:finite_fdp}
To quantify the phase transition, we establish a finite-sample, high-probability bound on the empirical False Discovery Proportion (FDP). Let $V_t^{\text{extra}} = V_t^{\text{DOMT}} - V_t^{\text{base}}$ be the additional false positives induced by exploration. Under the amplitude $\epsilon_s = \kappa \alpha / \sqrt{s}$, an application of the Azuma-Hoeffding inequality ensures that for any $\delta \in (0,1)$, with probability at least $1-\delta$:
\begin{equation}
    \text{FDP}_t^{\text{DOMT}} \le \text{FDP}_t^{\text{base}} + \frac{\kappa \alpha \sqrt{t} + \sqrt{\frac{t}{2} \ln(1/\delta)}}{\max(1, R_t^{\text{DOMT}})}.
\end{equation}
\end{theorem}

\begin{remark}[Statistical interpretation of the bound]
\label{rem:statistical_interpretation}
The bound rigorously characterizes the graceful degradation of DOMT. In stationary regimes where discoveries accumulate linearly ($R_t = \Theta(t)$), the FDP inflation penalty vanishes at an $\mathcal{O}(t^{-1/2})$ rate, recovering exact asymptotic safety. Conversely, in extreme adversarial pure-null settings, the structural absence of true signals dictates that the transient empirical FDP will inevitably hit $100\%$. However, Theorem \ref{thm:finite_fdp} reveals the true safety mechanism of the exploration: the \textit{absolute number} of additional false positives ($V_t^{\text{extra}}$) is strictly bounded by $\mathcal{O}(\sqrt{t})$. By restricting the localized cold-start tax to a sublinear growth rather than a catastrophic linear ($\Omega(t)$) accumulation, DOMT structurally prevents systemic failure, ensuring that the exploration remains unconditionally safe. 
Proofs are deferred to Appendix \ref{B-3}.
\end{remark}

\subsection{Regret decomposition of DOMT}

Beyond asymptotic safety, we next quantify the precise exploration overhead. Theorem \ref{thm:regret_gap} and Corollary \ref{cor:regret_asymptotic} prove that the weighted regret gap between DOMT and any chosen base procedure is strictly bounded by $\mathcal{O}(\sqrt{T})$. Consequently, our perturbation fundamentally preserves the optimal regret trajectory of the underlying algorithm.

\begin{theorem}[Weighted Regret gap]
\label{thm:regret_gap}
Let $\text{Regret}_T^{\text{base}}(a,b)$ and $\text{Regret}_T^{\text{DOMT}}(a,b)$ denote the expected weighted regrets of any chosen base procedure and its DOMT counterpart, respectively. Subject to the regularity conditions in Appendix \ref{B-1}, there exists a constant $K = a \kappa \alpha > 0$ (elegantly independent of the underlying signal distribution and the false negative penalty $b$) such that:
\begin{equation}
    \mathbb{E}[\text{Regret}_T^{\text{DOMT}}(a,b)] - \mathbb{E}[\text{Regret}_T^{\text{base}}(a,b)] \le K\sqrt{T}.
\end{equation}
\end{theorem}

\begin{corollary}[Asymptotic regret order]
\label{cor:regret_asymptotic}
Under the conditions of Theorem \ref{thm:regret_gap}, if the base procedure satisfies a minimum power condition such that $\mathbb{E}[\text{Regret}_T^{\text{base}}(a,b)] = \Omega(T)$, then the expected weighted regret of DOMT remains $\Theta(T)$. Specifically, the regret precisely satisfies $\mathbb{E}[\text{Regret}_T^{\text{DOMT}}(a,b)] = \mathbb{E}[\text{Regret}_T^{\text{base}}(a,b)] + \mathcal{O}(\sqrt{T})$.
\end{corollary}

\begin{remark}[Duality of Regret Conservation]
The $\mathcal{O}(\sqrt{T})$ bound in Theorem \ref{thm:regret_gap} strictly quantifies the worst-case exploration overhead. As detailed in the subsequent section, this exact mechanism enables a matching $\Omega(\sqrt{T})$ reduction in regret during extreme non-stationary cold starts. This profound symmetry constitutes an order-optimal mitigation of the threshold depletion trap inherent to deterministic baselines. The proofs for Theorem \ref{thm:regret_gap} and Corollary \ref{cor:regret_asymptotic} are deferred to Appendix \ref{B-4}.
\end{remark}

\section{The advantage of the Decoupled-OMT algorithm}
\label{sec:regret_reduction}

While Section \ref{sec:algorithm} established that the exploration overhead is strictly bounded by $\mathcal{O}(\sqrt{T})$ for any chosen base procedure, we now prove that this stochastic mechanism is not merely a safety constraint, but a powerful catalyst for performance enhancement. In non-stationary environments where signals appear in delayed clusters, the DOMT framework effectively circumvents the threshold depletion trap. It extracts a macroscopic $\Omega(\sqrt{T})$ reduction in weighted regret, thereby executing an order-optimal geometric traversal of the theoretical Pareto frontier.

Deterministic baselines are structurally vulnerable to such bursty environments, where a prolonged sequence of pure noise precedes a dense cluster of discoveries. We formalize this via a \textit{bursty signal model}: $Y_t = 0$ for an initial phase $t \in \{1, \dots, T_0\}$, followed by a stable mixture where $Y_t = 1$ occurs with non-vanishing probability for $t > T_0$. In these adversarial cold-start scenarios, unperturbed baselines suffer from absolute threshold depletion, driving $\lambda_t^{\text{base}}$ down to $\mathcal{O}(t^{-1})$. This asymptotically traps the deterministic procedure in a state of extreme power deficiency, rendering it structurally blind to the subsequent burst of genuine signals.

\begin{theorem}[Order-optimal regret reduction]
\label{thm:sqrt_reduction}
Under the bursty signal model with $T_0 = \Theta(T)$, let $\text{Regret}_T^{\text{base}}$ and $\text{Regret}_T^{\text{DOMT}}$ denote the expected weighted regrets of any chosen base procedure and its DOMT counterpart, respectively. For any bursty signal configuration in the weak-signal regime ($L\alpha < 1$) satisfying the conditions in Appendix \ref{B-1}, there exists a critical penalty threshold $M^* > 0$ such that, provided the asymmetric weight ratio satisfies $b/a > M^*$, the expected regret reduction satisfies:
\begin{equation}
    \mathbb{E}[\text{Regret}_T^{\text{base}}] - \mathbb{E}[\text{Regret}_T^{\text{DOMT}}] = \Omega(\sqrt{T}).
\end{equation}
\end{theorem}

\textit{Proof Sketch.} Let $\Delta M_T = \mathbb{E}[M_T^{\text{base}} - M_T^{\text{DOMT}}]$ and $\Delta V_T = \mathbb{E}[V_T^{\text{DOMT}} - V_T^{\text{base}}]$. For the null-dominant phase $t \le T_0 = \Theta(T)$, threshold depletion drives $\lambda_{t>T_0}^{\text{base}} = \mathcal{O}(t^{-1})$, pushing baseline discoveries asymptotically to zero and establishing $\mathbb{E}[M_T^{\text{base}}] = \Omega(T)$ under the weak-signal assumption. By introducing the non-negative perturbation $\xi_t \sim \text{Uniform}[0, \kappa \alpha / \sqrt{t}]$, DOMT establishes a unidirectional stochastic exploration lower bound. Applying the signal detectability lower bound ensures the existence of an environment-dependent constant $C_1 > 0$ such that the recovery of false negatives satisfies:
\begin{equation}
    \Delta M_T \ge \sum_{t=T_0+1}^T \mathbb{P}(Y_t=1) \cdot \mathbb{P}\Big(p_t \le \xi_t \mid Y_t=1\Big) \ge C_1 \sum_{t=T_0+1}^T \frac{\kappa \alpha}{\sqrt{t}} = \Omega(\kappa \alpha \sqrt{T}).
\end{equation}
Concurrently, the derivation of Theorem \ref{thm:regret_gap} bounds the expected exploration penalty as $\Delta V_T \le C_2 \cdot \kappa \alpha \sqrt{T}$ for an environment-dependent constant $C_2 > 0$. Synthesizing these dual bounds, the net weighted regret reduction is factored:
\begin{equation}
    \Delta \text{Regret}_T = b\Delta M_T - a\Delta V_T \ge \kappa \alpha (b C_1 - a C_2) \sqrt{T}.
\end{equation}
By explicitly defining the critical threshold as $M^* := C_2/C_1$, the precondition $b/a > M^*$ ensures that the core coefficient $(b C_1 - a C_2)$ is positive. The required threshold $M^*$ is independent of the exploration coefficient $\kappa$, which solely scales the magnitude of the $\Omega(\sqrt{T})$ reduction. The full derivation is deferred to Appendix \ref{C-1}.

\begin{remark}[Physical interpretation and the Cold-Start Tax]
\label{remark11}
The critical threshold $M^*$ encodes a profound physical law governing adversarial environments. Let $\rho = T_0/T \in (0,1)$ denote the burst delay ratio. Crucially, because both the false positive exploration tax and the false negative recovery scale proportionally with the amplitude $\kappa \alpha$, this algorithmic parameter perfectly cancels out in their ratio. Integrating the exact marginal rates over the pure-noise and bursty phases reveals an elegant closed-form solution for the environment-dependent threshold:
\begin{equation}
    M^*(\rho) = \underbrace{\frac{\pi}{\mu(1-\pi)}}_{\text{Stationary Difficulty}} + \underbrace{\frac{1}{\mu(1-\pi)} \frac{\sqrt{\rho}}{1-\sqrt{\rho}}}_{\text{Cold-Start Tax}}
\end{equation}
This equation strictly governs OMT investment. The first term captures the inherent difficulty of the underlying signal distribution. The second term, the ``Cold-Start Tax'', quantifies the pure-noise exploration penalty. As the delay $\rho \to 1$, this tax mathematically explodes to infinity. This proves that extreme adversarial delays justify exploration \textit{only under} an extreme aversion to false negatives ($b/a \to \infty$). The full derivation is deferred to Appendix \ref{C-2}.
\end{remark}

While the cold-start tax defines the strict boundary for profitability, the true fundamental advantage of DOMT crystallizes in highly asymmetric regimes ($b \gg a$). In ubiquitous high-stakes domains such as medical screening and fraud detection, missing a critical signal carries a catastrophic penalty. In these settings, the $\mathcal{O}(a\sqrt{T})$ exploration cost acts merely as a fixed, constrained investment in false positive risk. In stark contrast, the algorithm's discovery dividend scales dynamically with signal importance. We formally quantify this linear amplification of utility in the following Theorem \ref{thm:weight_amplification}.

\begin{theorem}[Weight-sensitive advantage]
\label{thm:weight_amplification}
Let $M = b/a > 1$ denote the asymmetric penalty ratio. Under the identical bursty signal model specified in Theorem \ref{thm:sqrt_reduction}, the expected regret reduction achieved by DOMT scales proportionally with both $M$ and the exploration amplitude $\kappa \alpha$:
\begin{equation}
    \mathbb{E}[\text{Regret}_T^{\text{base}}] - \mathbb{E}[\text{Regret}_T^{\text{DOMT}}] = \Omega(a M \kappa \alpha \sqrt{T}).
\end{equation}
\end{theorem}

\begin{remark}[The fallacy of $\mathcal{O}(\log T)$ exploration]
\label{remark12}
One might hypothesize that steeper decays (e.g., $\epsilon_t \propto t^{-1}$) could compress the exploration penalty to $\mathcal{O}(\log T)$. This is a structural trap. Because null $p$-values are standard uniform, the expected false positive tax depends exclusively on the first moment $\mathbb{E}[\xi_t]$, rendering tail-shape optimizations irrelevant. If one forces $\mathbb{E}[\xi_t] = \mathcal{O}(t^{-1})$ to achieve a logarithmic cost, the Lipschitz constraint on alternative signals simultaneously chokes discovery recovery to $\mathcal{O}(\log T)$. Against the baseline's massive $\Omega(T)$ deficit, an $\mathcal{O}(\log T)$ rescue is asymptotically negligible. 
Thus, the $\Theta(t^{-1/2})$ uniform envelope is not an arbitrary heuristic, but the critical structural rate required to unlock an order-optimal $\Omega(\sqrt{T})$ regret reduction while strictly averting finite-sample FDP spikes. 
The proof of Theorem \ref{thm:weight_amplification} and extended discussion of Remark \ref{remark12} are deferred to Appendix \ref{C-3} and Appendix \ref{C-4} respectively.
\end{remark}

\vspace{-3mm}
\paragraph{Geometric synthesis: navigating the Pareto Frontier.}
The theoretical contributions of DOMT can be unified through a geometric lens in the error space $(V_T, M_T)$. Recall from Section \ref{sec:impossibility} that deterministic mechanisms are confined by a fundamental trade-off, creating an ``L-shaped'' region of infeasibility where $V_T$ and $M_T$ cannot be simultaneously sublinear. This manifests as a vertical asymptote of discovery failure during non-stationary cold starts. As will be empirically validated in Section \ref{sec:experiments}, DOMT successfully navigates this restrictive boundary. By deliberately trading a bounded horizontal penalty for a macroscopic vertical reduction, DOMT reaches optimal trade-off states along the non-stationary Pareto frontier that remain otherwise inaccessible to unperturbed baselines. We formalize this theoretical trajectory as follows:

\begin{remark}[Geometric Pareto traversal]
\label{rem:pareto_traversal}
The DOMT framework executes an order-optimal traversal in the $(V_T, M_T)$ plane. By deliberately incurring a strictly bounded horizontal penalty $\Delta V_T = \mathcal{O}(\sqrt{T})$, it achieves a macroscopic vertical descent $\Delta M_T = \Omega(\sqrt{T})$. This precise geometric translation effectively bypasses the deterministic L-shaped trap, allowing the procedure to safely access the non-stationary Pareto frontier.
\end{remark}

\section{Experiments}
\label{sec:experiments}

We evaluate the DOMT framework across three classical deterministic baselines (LOND, LORD, SAFFRON) \cite{javanmard2018online, ramdas2018saffron} under generalized empirical distributions. \textit{Crucially, while our theory assumes strict adversarial constraints to establish worst-case bounds, these deliberately relaxed simulations demonstrate DOMT's robust real-world applicability.} Performance is tracked via empirical FDR, Statistical Power, and our proposed Weighted Regret.
See Appendix \ref{E-1} for more detailed settings.

As visualized in Figure \ref{fig:main_dynamics}, DOMT maintains strict FDR control in stationary environments (Top Row). However, in the initial pure-null phase of the bursty environment (Bottom Row), any exploratory rejection leads to a transient FDP of 100\% due to the structural absence of true signals. Consequently, instead of displaying an uninformative FDR curve for the bursty phase, we highlight the Threshold Dynamics (bottom-left) to expose the underlying physical mechanism: while paying a localized 'Cold-Start Tax' in errors, DOMT successfully prevents the absolute threshold depletion that cripples deterministic baselines, thereby securing a macroscopic discovery dividend once the signal burst arrives.
Remarkably, this macroscopic advantage persists even in stationary settings, a phenomenon we attribute to the probabilistic nature of testing sequences:

\begin{remark}[Robustness to local mini-droughts]
\label{remark14}
DOMT's stationary power advantage is probabilistically inevitable. By the Borel-Cantelli lemma, local ``mini-droughts'' (consecutive pure nulls) occur almost surely as $T \to \infty$, causing deterministic baselines to momentarily falter. DOMT's stochastic envelope effortlessly bridges these micro-depletions, capturing signals that unperturbed procedures structurally miss. Formal average-case analysis is deferred to Appendix \ref{D}.
\end{remark}

\vspace{-2mm}
\begin{figure}[htbp]
    \centering
    \includegraphics[width=0.8\textwidth]{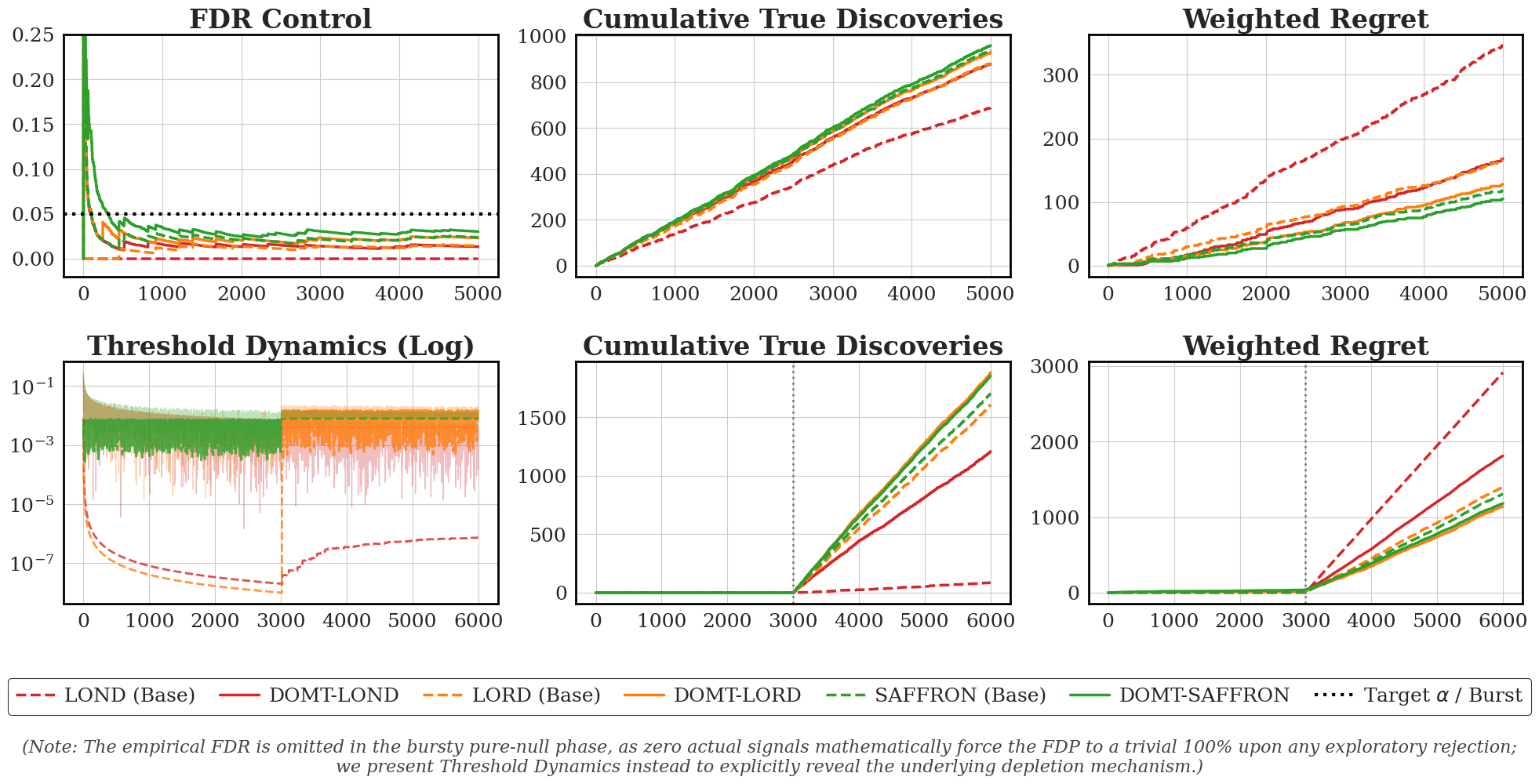}
    \caption{Performance Dynamics across Environments. \textbf{Top Row (Stationary):} DOMT maintains strict FDR control while delivering consistently enhanced power. \textbf{Bottom Row (Bursty):} Stochastic exploration mitigates absolute threshold depletion, effectively capturing delayed signal bursts. 
    % \protect\footnotemark[1]
    }
    \label{fig:main_dynamics}
\end{figure}
% \footnotetext[1]{ Code available at 
% % \url{https://anonymous.4open.science/r/domt_neurips26-589F}.}
% \url{https://anonymous.4open.science/r/domt-neurips26_26172}.}
% \vspace{-4mm}

\vspace{-1mm}

We conclude by evaluating regret phase transitions and geometric Pareto optimality (Figure \ref{fig:pareto_phase}). The left panel maps weighted regret across varying penalty ratios $b/a$ and burst delays $\rho$. Corroborating Remark \ref{rem:pareto_traversal}, DOMT establishes robust superiority once the asymmetry penalty exceeds the theoretical cold-start boundary. Projecting these outcomes onto the $(V_T, M_T)$ space (Right panel) confirms that unperturbed baselines are structurally trapped in the extreme vertical asymptote. By paying a strictly bounded horizontal tax in false discoveries, DOMT successfully secures a macroscopic vertical reduction in false negatives. 
Extensive supplementary validations are deferred to Appendix \ref{E}.

\vspace{-1mm}
\begin{figure}[htbp]
    \centering
    \begin{minipage}[b]{0.35\textwidth}
        \centering
        \includegraphics[width=\linewidth]{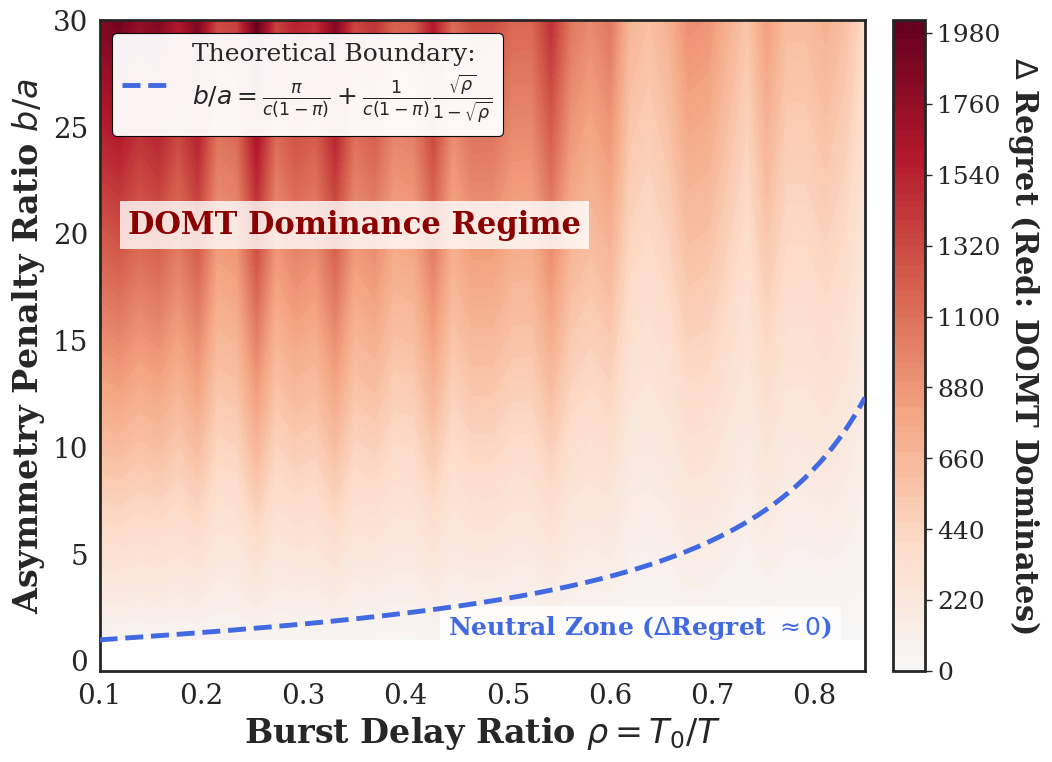}
    \end{minipage}%
    \hspace{0.07\textwidth}
    \begin{minipage}[b]{0.35\textwidth}
        \centering
        \includegraphics[width=\linewidth]{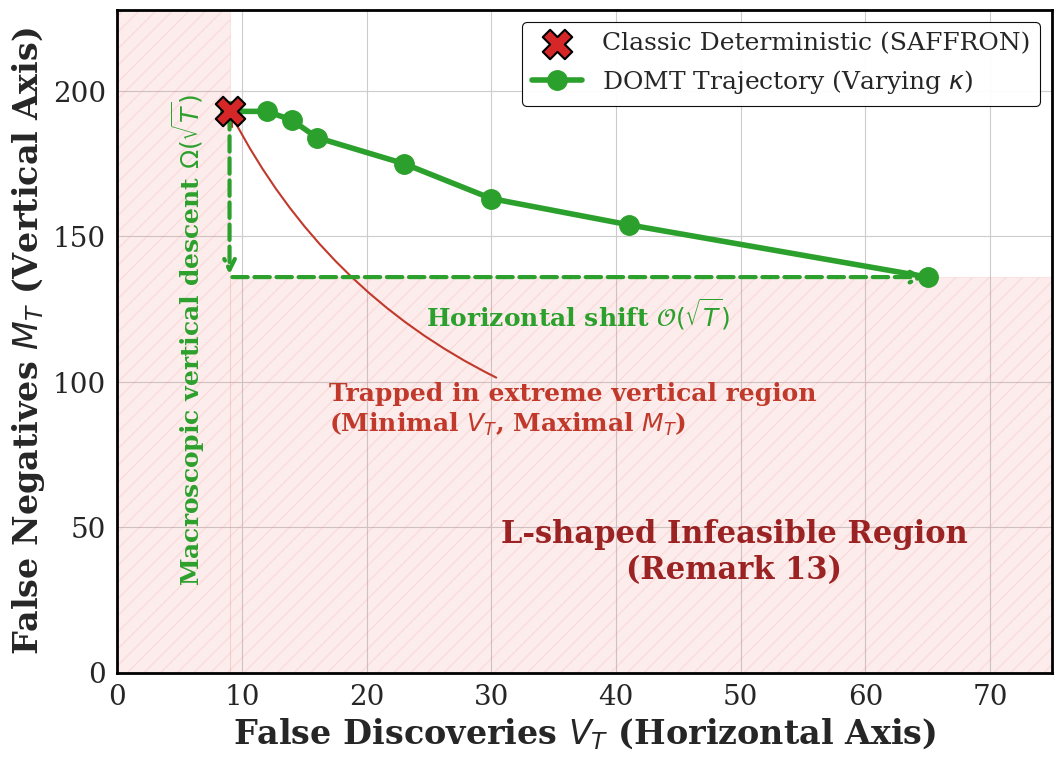}
    \end{minipage}
    \vspace{-1mm}
    \caption{Regret Phase Transition and Geometric Synthesis. \textbf{Left:} Asymmetric regret scaling, affirming the theoretical cold-start boundary. \textbf{Right:} Empirical Pareto traversal escaping the lock-in.}
    \label{fig:pareto_phase}
\end{figure}
\vspace{-1mm}

\section{Conclusion}

\vspace{-1mm}
This work introduces the \textit{Weighted Regret} perspective to evaluate online multiple testing, examines existing procedures under this lens, and proposes Decoupled-OMT (DOMT) as an approach that offers targeted improvements, substantiated by rigorous theoretical derivations and empirical validations. 
Anticipating that this perspective may inspire more research within the field, we formally defer the discussion to Appendix \ref{F}, examining how current limitations naturally motivate future trajectories.

\bibliographystyle{plain}
\bibliography{refs}

@article{johari2022experimental,
author = {Johari, Ramesh and Li, Hannah and Liskovich, Inessa and Weintraub, Gabriel Y.},
title = {Experimental Design in Two-Sided Platforms: An Analysis of Bias},
journal = {Management Science},
volume = {68},
number = {10},
pages = {7069-7089},
year = {2022},
doi = {10.1287/mnsc.2021.4247}
}

@article{benjamini1995controlling,
  title={Controlling the false discovery rate: a practical and powerful approach to multiple testing},
  author={Benjamini, Yoav and Hochberg, Yosef},
  journal={Journal of the Royal Statistical Society: Series B (Methodological)},
  volume={57},
  number={1},
  pages={289--300},
  year={1995}
}

@article{foster2008alpha,
  title={$\alpha$-investing: a procedure for sequential control of expected false discoveries},
  author={Foster, Dean P and Stine, Robert A},
  journal={Journal of the Royal Statistical Society: Series B (Statistical Methodology)},
  volume={70},
  year={2008}
}

@article{javanmard2018online,
author = {Adel Javanmard and Andrea Montanari},
title = {{Online rules for control of false discovery rate and false discovery exceedance}},
volume = {46},
journal = {The Annals of Statistics},
number = {2},
pages = {526 -- 554},
year = {2018}
}

@inproceedings{ramdas2018saffron,
  title={SAFFRON: an adaptive algorithm for online control of the false discovery rate},
  author={Aaditya Ramdas and Tijana Zrnic and Martin J. Wainwright and Michael I. Jordan},
  booktitle={ICML},
  year={2018}
}

@article{lin2024online,
  author       = {Ruicheng Ao and
                  Hongyu Chen and
                  David Simchi{-}Levi and
                  Feng Zhu},
  title        = {Online Local False Discovery Rate Control: {A} Resource Allocation
                  Approach},
  journal      = {CoRR},
  volume       = {abs/2402.11425},
  year         = {2024}
}

@article{Kuang2026SCOREAU,
  title={SCORE: A Unified Framework for Overshoot Refund in Online FDR Control},
  author={Kuang, Qi and Gang, Bowen and Xia, Yin},
  journal={arXiv preprint arXiv:2601.20386},
  year={2026}
}

@inproceedings{NIPS2017_7f018eb7,
 author = {Ramdas, Aaditya and Yang, Fanny and Wainwright, Martin J and Jordan, Michael I},
 booktitle = {NeurIPS},
 editor = {I. Guyon and U. Von Luxburg and S. Bengio and H. Wallach and R. Fergus and S. Vishwanathan and R. Garnett},
 title = {Online control of the false discovery rate with decaying memory},
 year = {2017}
}

@article{Delecroix2023MonitoringRI,
  title={Monitoring resilience in bursts},
  author={Clara Delecroix and Egbert H. van Nes and Martin W. Scheffer and Ingrid A van de Leemput},
  journal={Proceedings of the National Academy of Sciences of the United States of America},
  year={2023},
  volume={121}
}

@article{Savage1951TheTO,
  title={The Theory of Statistical Decision},
  author={Leonard J. Savage},
  journal={Journal of the American Statistical Association},
  year={1951},
  volume={46},
  pages={55-67}
}

@article{Lai1985AsymptoticallyEA,
  title={Asymptotically efficient adaptive allocation rules},
  author={Tze Leung Lai and Herbert E. Robbins},
  journal={Advances in Applied Mathematics},
  year={1985},
  volume={6},
  pages={4-22}
}

@inproceedings{Yang2017AFF,
  title={A framework for Multi-A(rmed)/B(andit) Testing with Online FDR Control},
  author={Fanny Yang and Aaditya Ramdas and Kevin G. Jamieson and Martin J. Wainwright},
  booktitle={NeurIPS},
  year={2017}
}

@inproceedings{Chen2019ContextualOF,
  title={Contextual Online False Discovery Rate Control},
  author={Shiyun Chen and Shiva Prasad Kasiviswanathan},
  booktitle={AISTATS},
  year={2019}
}

@inproceedings{Xu2023OnlineMT,
  title={Online multiple testing with e-values},
  author={Ziyu Xu and Aaditya Ramdas},
  booktitle={AISTATS},
  year={2023}
}

@inproceedings{DBLP:conf/iclr/PodkopaevR22,
  author       = {Aleksandr Podkopaev and
                  Aaditya Ramdas},
  title        = {Tracking the risk of a deployed model and detecting harmful distribution
                  shifts},
  booktitle    = {ICLR},
  year         = {2022}
}

@inproceedings{Angelopoulos2022ConformalRC,
  title={Conformal Risk Control},
  author={Anastasios Nikolas Angelopoulos and Stephen Bates and Adam Fisch and Lihua Lei and Tal Schuster},
  booktitle={ICLR},
  year={2024}
}

@inproceedings{tian2019addis,
  title={{ADDIS}: an adaptive discarding algorithm for online {FDR} control with conservative nulls},
  author={Tian, Jinjin and Ramdas, Aaditya},
  booktitle={NeurIPS},
  year={2019}
}

@article{Groza2022TheIM,
  title={The International Mouse Phenotyping Consortium: comprehensive knockout phenotyping underpinning the study of human disease},
  author={Tudor Groza and Federico L{\'o}pez-G{\'o}mez and Hamed Haseli Mashhadi and Violeta Mu{\~n}oz-Fuentes and Osman Nuri G{\"u}nes and Robert J. Wilson and Pilar Cacheiro and Anthony Frost and Piia Keskivali-Bond and Bora Vardal and Aaron McCoy and Tsz Kwan Cheng and Luis A. Santos and Sara E. Wells and Damian Smedley and Ann-Marie Mallon and Helen E. Parkinson},
  journal={Nucleic Acids Research},
  year={2022},
  volume={51},
  pages={D1038 - D1045}
}

@article{Singh2002GeneEC,
  title={Gene expression correlates of clinical prostate cancer behavior.},
  author={Dinesh Singh and Phillip G. Febbo and Kenneth N. Ross and Donald G. Jackson and Judith B. Manola and Christine Ladd and Pablo Tamayo and Andrew A. Renshaw and Anthony V. D'Amico and Jerome P. Richie and Eric S. Lander and Massimo Loda and Philip W. Kantoff and Todd R. Golub and William R. Sellers},
  journal={Cancer cell},
  year={2002},
  volume={1 2},
  pages={
          203-9
        }
}

@article{Bates2021DistributionFreeRP,
  title={Distribution-Free, Risk-Controlling Prediction Sets},
  author={Stephen Bates and Anastasios Nikolas Angelopoulos and Lihua Lei and Jitendra Malik and Michael I. Jordan},
  journal={J. ACM},
  year={2021},
  volume={68},
  pages={43:1-43:34}
}

@article{Lee2026SyntheticPoweredMT,
  title={Synthetic-Powered Multiple Testing with FDR Control},
  author={Yonghoon Lee and Meshi Bashari and Edgar Dobriban and Yaniv Romano},
  journal={ArXiv},
  year={2026},
  volume={abs/2602.16690}
}

@article{liu2026sequential,
  title={Sequential Multiple Testing: A Second-Order Asymptotic Analysis},
  author={Liu, Jingyu and Song, Yanglei},
  journal={arXiv preprint arXiv:2603.04685},
  year={2026}
}

@article{Kale2025LieTM,
  title={Lie to Me: Knowledge Graphs for Robust Hallucination Self-Detection in LLMs},
  author={Sahil Kale and Antonio Luca Alfeo},
  journal={ArXiv},
  year={2025},
  volume={abs/2512.23547}
}

@article{Zhao2023FalseDR,
  title={False Discovery Rate Control For Structured Multiple Testing: Asymmetric Rules And Conformal Q-values},
  author={Zinan Zhao and Wenguang Sun},
  journal={Journal of the American Statistical Association},
  year={2023},
  volume={120},
  pages={805 - 817}
}

@inproceedings{leebandit,
  title={From Bandit Regret to FDR Control: Online Selective Generation with Feedback Unlocking},
  author={Lee, Minjae and Jung, Yoonjae and Park, Sangdon},
  booktitle={Agentic AI in the Wild: From Hallucinations to Reliable Autonomy},
  year={2026}
}

@article{Chiong2023TheMD,
  title={The Minimax-Regret Decision Framework for Online A/B Tests},
  author={Khai Xiang Chiong and Joonhwi Joo},
  journal={SSRN Electronic Journal},
  year={2023}
}

@inproceedings{xu2024more,
  title={More powerful multiple testing under dependence via randomization},
  author={Xu, Ziyu and Ramdas, Aaditya},
  booktitle={AISTATS},
  year={2024}
}

@article{fischer2024online,
  title={An online generalization of the (e-) Benjamini-Hochberg procedure},
  author={Fischer, Lasse and Xu, Ziyu and Ramdas, Aaditya},
  journal={arXiv preprint arXiv:2407.20683},
  year={2024}
}

@article{Robertson2022OnlineMH,
  title={Online multiple hypothesis testing},
  author={David S. Robertson and James M. S. Wason and Aaditya Ramdas},
  journal={Statistical science : a review journal of the Institute of Mathematical Statistics},
  year={2022},
  volume={38},
  pages={557 - 575}
}

@article{zrnic2021asynchronous,
  title={Asynchronous online testing of multiple hypotheses},
  author={Zrnic, Tijana and Ramdas, Aaditya and Jordan, Michael I},
  journal={Journal of Machine Learning Research},
  volume={22},
  number={33},
  pages={1--39},
  year={2021}
}

@article{Gang2020StructureAdaptiveST,
  title={Structure–Adaptive Sequential Testing for Online False Discovery Rate Control},
  author={Bowen Gang and Wenguang Sun and Weinan Wang},
  journal={Journal of the American Statistical Association},
  year={2020},
  volume={118},
  pages={732 - 745}
}

@inproceedings{gibbs2021adaptive,
  title={Adaptive conformal inference under distribution shift},
  author={Gibbs, Isaac and Candes, Emmanuel},
  booktitle={NeurIPS},
  year={2021}
}

@inproceedings{bastani2022practical,
  title={Practical adversarial multivalid conformal prediction},
  author={Bastani, Osbert and Gupta, Varun and Jung, Christopher and Noarov, Georgy and Ramalingam, Ramya and Roth, Aaron},
  booktitle={NeurIPS},
  year={2022}
}

@article{fischer2024closure,
  title={The online closure principle},
  author={Fischer, Lasse and Bofill Roig, Marta and Brannath, Werner},
  journal={The Annals of Statistics},
  volume={52},
  number={2},
  pages={817--841},
  year={2024}
}

@inproceedings{Zhang2025eGAIEG,
  title={e-GAI: e-value-based Generalized $\alpha$-Investing for Online False Discovery Rate Control},
  author={Yifan Zhang and Zijian Wei and Haojie Ren and Changliang Zou},
  booktitle={ICML},
  year={2025}
}

@inproceedings{xu2021unified,
  title={A unified framework for bandit multiple testing},
  author={Xu, Ziyu and Ramdas, Aaditya},
  booktitle={NeurIPS},
  year={2021}
}

@article{dohler2024online,
  title={Online multiple testing with super-uniformity reward},
  author={D{\"o}hler, Sebastian and Meah, Iqraa and Roquain, Etienne},
  journal={Electronic Journal of Statistics},
  volume={18},
  number={1},
  pages={1293--1354},
  year={2024}
}

@inproceedings{xu2022dynamic,
  title={Dynamic algorithms for online multiple testing},
  author={Xu, Ziyu and Ramdas, Aaditya},
  booktitle={Mathematical and Scientific Machine Learning},
  pages={955--986},
  year={2022}
}

@article{tian2021online,
  title={Online control of the familywise error rate},
  author={Tian, Jinjin and Ramdas, Aaditya},
  journal={Statistical Methods in Medical Research},
  volume={30},
  number={4},
  pages={976--993},
  year={2021}
}

@article{robertson2022online,
  title={Online error rate control for platform trials},
  author={Robertson, David S and Wason, James MS and K{\"o}nig, Franz and Posch, Martin and Jaki, Thomas},
  journal={Statistics in Medicine},
  volume={42},
  number={14},
  pages={2475--2495},
  year={2023}
}

@article{bao2024cap,
  title={CAP: A General Algorithm for Online Selective Conformal Prediction with FCR Control},
  author={Bao, Yajie and Huo, Yuyang and Ren, Haojie and Zou, Changliang},
  journal={Journal of Machine Learning Research},
  volume={26},
  number={287},
  pages={1--74},
  year={2025}
}

@article{shaffer1995multiple,
  title={Multiple hypothesis testing},
  author={Shaffer, Juliet Popper and others},
  journal={Annual review of psychology},
  volume={46},
  number={1},
  pages={561--584},
  year={1995}
}

@book{efron2012large,
  title={Large-scale inference: empirical Bayes methods for estimation, testing, and prediction},
  author={Efron, Bradley},
  volume={1},
  year={2012},
  publisher={Cambridge University Press}
}
%%%%%%%%%%%%%%%%%%%%%%%%%%%%%%%%%%%%%%%%%%%%%%%%%%%%%%%%%%%%

\newpage

\appendix

\phantomsection
\section*{Organization of the Appendix}
\label{apporg}

To facilitate navigation and comprehensive review, the appendix are organized as follows:

\begin{itemize} %[leftmargin=*]
    
    \item \textbf{Appendix \ref{F-rlw} (Related Work):} It positions our Weighted Regret perspective and decoupling mechanism within the broader methodological lineage of OMT.
    
    \item \textbf{Appendix \ref{A} (Deferred Proofs for Section \ref{sec:impossibility}):} Establishes the fundamental limits of deterministic OMT. It details the proofs for the linear regret lower bound of constant power (\ref{A-1}) and the inevitable false negative penalty of threshold decay (\ref{A-2}). We further instantiate these limits on canonical baselines (\ref{A-3}), culminating in the proof of the impossibility of simultaneous sublinear regrets (\ref{A-4}).

    \item \textbf{Appendix \ref{B} (Deferred Proofs for Section \ref{sec:algorithm}):} Provides the rigorous mathematical foundation for DOMT. Following the specification of necessary regularity conditions (\ref{B-1}), it presents the complete proofs for DOMT's exact asymptotic FDR control (\ref{B-2}), its finite-sample FDP inflation bounds (\ref{B-3}), and the strict $\mathcal{O}(\sqrt{T})$ worst-case regret gap (\ref{B-4}).

    \item \textbf{Appendix \ref{C} (Deferred Proofs for Section \ref{sec:regret_reduction}):} Demonstrates the macroscopic advantages of DOMT in adversarial environments. It includes the rigorous derivation of the order-optimal $\Omega(\sqrt{T})$ regret reduction (\ref{C-1}) and the closed-form formulation of the ``Cold-Start Tax'' (\ref{C-2}). It also formally proves the framework's weight-sensitive advantage (\ref{C-3}) and provides an extended discussion refuting the structural fallacy of $\mathcal{O}(\log T)$ exploration (\ref{C-4}).

    \item \textbf{Appendix \ref{D} (Theoretical Extensions):} Formalizes the theoretical mechanisms behind our empirical observations. It mathematically quantifies DOMT's robust threshold advantage during local pure-null mini-droughts (\ref{D-1}) and rigorously proves its structural capability to preserve the super-martingale safety of modern $e$-value procedures (\ref{D-2}).

    \item \textbf{Appendix \ref{E} (Empirical Evaluations):} Details our experimental methodology (\ref{E-1}) and provides extensive supplementary validations. It features ablation studies validating the necessity of causal decoupling (\ref{E-2}) and confirming parameter universality (\ref{E-3}). We further demonstrate DOMT's compatibility with advanced methods (\ref{E-4}), and evaluate its practical efficacy across diverse real-world datasets including RNA-Seq microarrays, S\&P 500 anomalies, and IMPC phenotype data (\ref{E-5}).

    \item \textbf{Appendix \ref{F} (Extended Discussions):} Expands on the contextual and theoretical boundaries of our work. It provides candid analyses of DOMT's limitations, such as wealth deprivation (\ref{F-2}), and outlines actionable future directions towards autonomous, context-aware online testing (\ref{F-3}).
\end{itemize}

\section{More Discussion on Related Work}
\label{F-rlw}

To rigorously situate the Decoupled-OMT (DOMT) framework within the broader landscape of statistical inference, we systematically decompose the evolution of existing methodologies and clarify the boundaries of our proposed paradigm.

\subsection{From Constrained Optimization to Unconstrained Regret Minimization}
The idea of incorporating asymmetric importance into OMT has roots in the penalty-weighted FDR framework \cite{NIPS2017_7f018eb7}, which introduced per-hypothesis penalty weights to reflect differential importance across tests. However, this framework remains within the constrained optimization paradigm, aiming to maximize power subject to a weighted-FDR constraint.

More recently, online Local FDR control has been formulated as an online knapsack problem, achieving $\mathcal{O}(\log^2 T)$ regret under Local FDR oracles \cite{lin2024online}. It is important to note that their $\mathcal{O}(\log^2 T)$ bound and our $\Omega(\sqrt{T})$ bound operate under strictly different informational regimes. Specifically, \cite{lin2024online} assumes access to posterior probabilities, whereas our framework operates in the canonical Global FDR setting using only raw $p$-values. Therefore, these bounds reflect the distinct computational capacities of their respective input spaces rather than a direct methodological comparison.

Concurrently, active sequential paradigms have unified regret minimization with error control. Frameworks such as Multi-Armed Bandit (MAB) testing \cite{Yang2017AFF, xu2021unified} and online selective generation \cite{leebandit} establish rigorous connections between bandit regret and FDR. However, a critical distinction must be made regarding the source of data. These active frameworks operate under settings where the algorithm strategically dictates interventions (e.g., arm pulls or generation policies). In contrast, our work addresses passive testing streams where the arrival of hypotheses is strictly exogenous. 

Against this backdrop, our \textit{Weighted Regret} metric represents a deliberate shift from constrained to unconstrained optimization in passive OMT. Rather than treating FDR as a hard constraint and power as a secondary objective, we directly minimize the cumulative asymmetric loss relative to an omniscient Oracle. This shift enables the $\Omega(T)$ lower-bound proofs and order-optimal regret-reduction guarantees that are unavailable in constrained frameworks, aligning naturally with the modern risk-aware philosophy of Conformal Risk Control (CRC) \cite{Angelopoulos2022ConformalRC, Bates2021DistributionFreeRP}.

\subsection{The Evolution of Threshold Depletion Countermeasures}
The genesis of OMT is deeply rooted in $\alpha$-wealth allocation \cite{foster2008alpha}, formally unified into deterministic frameworks like LOND and LORD \cite{javanmard2018online, NIPS2017_7f018eb7}, and enhanced by adaptive mechanisms like SAFFRON \cite{ramdas2018saffron} and ADDIS \cite{tian2019addis}. To counter the ``alpha-death'' inherent in macroscopic cold-starts \cite{Delecroix2023MonitoringRI}, various mechanisms have emerged. The SURE framework exploits conservative null structures to effectively reward the algorithm \cite{dohler2024online}, while SupLORD employs dynamic scheduling to adjust testing levels based on wealth \cite{xu2022dynamic}. 

Concurrently, structure-adaptive methods like SAST \cite{Gang2020StructureAdaptiveST} utilize sliding windows to estimate local prior probabilities. However, in the extreme adversarial cold-start scenario modeled in our work, estimation-based methods face inherent limitations. Due to the absence of historical signal bursts during the pure-null phase, they lack the necessary data to calibrate thresholds, naturally limiting their reaction speed compared to model-free methods.

Recently, the field has witnessed a profound shift toward betting martingales utilizing $e$-values, exemplified by e-LOND \cite{Xu2023OnlineMT}, SCORE \cite{Kuang2026SCOREAU}, and $e$-GAI \cite{Zhang2025eGAIEG}. Simultaneously, theoretical boundaries have been expanded via the online closure principle \cite{fischer2024closure, fischer2024online}. To enhance power, recent stochastic extensions like U$e$-LOND \cite{xu2024more} share a similar randomized exploration spirit with our work. 

However, the critical difference lies in state management. Directly coupling exploration noise into the wealth state makes algorithms vulnerable to spurious cascades and irreversible threshold depletion during extreme droughts. DOMT, conversely, acts as a cross-paradigm meta-wrapper. By executing a non-negative stochastic expansion and maintaining strict \textit{causal decoupling}, it quarantines the noise from the virtual baseline. It does not replace $p$-value or $e$-value procedures but rather encapsulates them, safely shattering the zero-threshold deadlock while rigorously preserving their underlying super-martingale properties (Appendix \ref{D-2}).

\subsection{Emerging Horizons and Complex Topologies}
Looking forward, the online multiple testing landscape is rapidly expanding into highly complex, data-driven domains. A notable recent advancement is the integration of generative models and auxiliary historical information into the testing pipeline. For instance, recent synthetic-powered multiple testing frameworks \cite{Lee2026SyntheticPoweredMT} explore how to safely leverage synthetic data to enhance sample efficiency while preserving distribution-free FDR control. Concurrently, the proliferation of large language models (LLMs) has catalyzed the demand for online selective inference to strictly control generation hallucinations under non-stationary user interactions \cite{Kale2025LieTM, bao2024cap}. Furthermore, from a pure theoretical statistics perspective, there is a burgeoning push towards establishing second-order asymptotic optimality \cite{liu2026sequential}, handling arbitrary structured dependencies such as spatial or graphical models \cite{Zhao2023FalseDR}, and extending online error control to platform trials \cite{robertson2022online} and familywise error rates \cite{tian2021online}. 

DOMT's core philosophy---causally decoupling active exploration from the primary testing martingale---holds profound potential across these frontiers. By mathematically quarantining exploratory noise, DOMT could enable the aggressive utilization of potentially biased synthetic priors or selective LLM decoding strategies, without compromising the rigid martingale safety boundaries required in these sophisticated topologies.

\section{Deferred Proofs and Derivations for Section \ref{sec:impossibility}}
\label{A}

\subsection{Proof of Theorem \ref{thm:power_lower_bound} (Weighted Linear Regret Lower Bound)}
\label{A-1}

\begin{proof}
For the purpose of rigorous martingale analysis, we define an omniscient filtration $\mathcal{G}_t = \mathcal{F}_t \vee \sigma(Y_1, \dots, Y_t)$ which explicitly incorporates the unobserved true states. Since the true state generation and $p$-values are conditionally independent of past decisions, all derived conditional expectations hold validly, and the subsequently constructed $X_t$ forms a strict martingale difference sequence with respect to $\mathcal{G}_t$.

We restate the Weighted Regret for an online testing algorithm over $T$ rounds:
\begin{equation}
\text{Regret}_T(a,b) = a \cdot V_T + b \cdot M_T = a \sum_{t=1}^T \mathbf{1}\{Y_t=0, \delta_t=1\} + b \sum_{t=1}^T \mathbf{1}\{Y_t=1, \delta_t=0\}.
\end{equation}
Given that the weights $a, b > 0$ and the number of false negatives $M_T$ is non-negative, the following lower bound holds unconditionally:
\begin{equation}
\text{Regret}_T(a,b) \ge a \cdot V_T = a \sum_{t=1}^T \mathbf{1}\{Y_t=0, p_t \le \lambda_t\}.
\end{equation}
To evaluate the expected regret, we take the expectation of the total false discoveries $V_T$:
\begin{equation}
\mathbb{E}[V_T] = \sum_{t=1}^T \mathbb{E}[\mathbf{1}\{Y_t=0, p_t \le \lambda_t\}] = \sum_{t=1}^T \mathbb{P}(Y_t=0, p_t \le \lambda_t).
\end{equation}
In a stationary environment, the true states are independent and follow $Y_t \sim \text{Bernoulli}(1-\pi)$, implying $\mathbb{P}(Y_t=0) = \pi$. By the sequential decision protocol, the threshold $\lambda_t$ is $\mathcal{F}_{t-1}$-measurable (and thus $\mathcal{G}_{t-1}$-measurable). Since the null $p$-value $p_t$ is independent of the history $\mathcal{G}_{t-1}$ and follows a standard uniform distribution $U[0,1]$, we have:
\begin{equation}
\begin{aligned}
\mathbb{P}(Y_t=0, p_t \le \lambda_t) &= \mathbb{E}\left[ \mathbb{P}(Y_t=0, p_t \le \lambda_t \mid \mathcal{G}_{t-1}) \right] \\
&= \mathbb{E}\left[ \mathbb{P}(Y_t=0 \mid \mathcal{G}_{t-1}) \cdot \mathbb{P}(p_t \le \lambda_t \mid Y_t=0, \mathcal{G}_{t-1}) \right] \\
&= \mathbb{E}[\pi \cdot \lambda_t].
\end{aligned}
\end{equation}
The algorithm's constraint of maintaining a minimum power $\beta > 0$ necessitates $\lambda_t \ge \lambda_\beta := G^{-1}(\beta) > 0$ almost surely for all $t$. Substituting this threshold lower bound into the expectation:
\begin{equation}
\mathbb{E}[V_T] = \sum_{t=1}^T \pi \mathbb{E}[\lambda_t] \ge \sum_{t=1}^T \pi \lambda_\beta = \pi \lambda_\beta T.
\end{equation}
Combining these results, the expected weighted regret satisfies:
\begin{equation}
\mathbb{E}[\text{Regret}_T(a,b)] \ge a \cdot \mathbb{E}[V_T] \ge a \pi \lambda_\beta T = \Omega(T).
\end{equation}

To establish the almost sure (a.s.) limit, we define the centered sequence $X_t = \mathbf{1}\{Y_t=0, p_t \le \lambda_t\} - \pi \lambda_t$. Note that $\mathbb{E}[X_t \mid \mathcal{G}_{t-1}] = 0$, forming a strict sequence of bounded martingale differences with respect to $\mathcal{G}_t$. By the Strong Law of Large Numbers for martingales (e.g., via the Azuma-Hoeffding inequality and the Borel-Cantelli lemma), we have $\frac{1}{T} \sum_{t=1}^T X_t \xrightarrow{a.s.} 0$ as $T \rightarrow \infty$. Consequently:
\begin{equation}
\liminf_{T \rightarrow \infty} \frac{1}{T} V_T = \liminf_{T \rightarrow \infty} \left( \frac{1}{T} \sum_{t=1}^T \pi \lambda_t + \frac{1}{T} \sum_{t=1}^T X_t \right) \ge \pi \lambda_\beta \quad \text{a.s.}
\end{equation}
Thus, $\liminf_{T \rightarrow \infty} \frac{1}{T} \text{Regret}_T(a,b) \ge a \pi \lambda_\beta$ holds almost surely along individual sample paths.
\end{proof}

\subsection{Proof of Theorem \ref{thm:fdr_barrier} (The Inevitable False Negative Penalty of Threshold Decay)}
\label{A-2}

\begin{proof}
For any deterministic online FDR algorithm, define $\tilde{\lambda}_t$ as its worst-case threshold sequence when exactly zero rejections have occurred in the history up to step $t-1$. To ensure asymptotic safety, this base sequence must satisfy $\sum_{t=1}^\infty \tilde{\lambda}_t \le C < 1$.

Consider the adversarial bursty environment: $Y_t = 0$ for $1 \le t \le T/2$, and $Y_t = 1$ for $T/2 < t \le T$. Let $\mathcal{E}_{\text{null}}$ denote the event of exactly zero rejections in the first phase. By the union bound:
\begin{equation}
\mathbb{P}(\mathcal{E}_{\text{null}}) \ge 1 - \sum_{t=1}^{T/2} \mathbb{P}(p_t \le \tilde{\lambda}_t \mid Y_t=0) = 1 - \sum_{t=1}^{T/2} \tilde{\lambda}_t \ge 1 - C > 0.
\end{equation}

Conditioned on $\mathcal{E}_{\text{null}}$, the algorithm enters the bursty phase with no historical discoveries. Its testing threshold will deterministically follow $\tilde{\lambda}_t$ until the first discovery is made. We evaluate the conditional probability of missing \textit{all} alternative signals in the second phase. Applying the Lipschitz bound $G(x) \le Lx$:
\begin{equation}
\mathbb{P}(\text{miss all signals in Phase 2} \mid \mathcal{E}_{\text{null}}) = \prod_{t=T/2+1}^T (1 - G(\tilde{\lambda}_t)) \ge \prod_{t=T/2+1}^T (1 - L\tilde{\lambda}_t)_+.
\end{equation}
If the algorithm misses all signals in Phase 2, the number of false negatives is exactly $M_T = T/2$. Therefore, the unconditional expected number of false negatives is strictly lower-bounded by:
\begin{equation}
\mathbb{E}[M_T] \ge \frac{T}{2} \cdot \mathbb{P}(\mathcal{E}_{\text{null}}) \cdot \prod_{t=T/2+1}^T (1 - L\tilde{\lambda}_t)_+ \ge \frac{T}{2} (1-C) \prod_{t=T/2+1}^T (1 - L\tilde{\lambda}_t)_+.
\end{equation}
Because the sequence $\tilde{\lambda}_t$ is summable ($\sum_{t=1}^\infty \tilde{\lambda}_t \le C$), we inherently have $\tilde{\lambda}_t \to 0$. For sufficiently large $T$, $L\tilde{\lambda}_t < 1$ holds for all $t > T/2$, ensuring the infinite product $\prod_{t=T/2+1}^\infty (1 - L\tilde{\lambda}_t)_+$ converges to a strictly positive constant $c > 0$. Consequently, $\mathbb{E}[M_T] \ge \frac{T}{2} (1-C) c = \Omega(T)$, establishing the inevitable linear false-negative penalty.
\end{proof}

\subsection{Regret Analysis for Canonical Deterministic Baselines}
\label{A-3}

To concretely instantiate the linear regret barrier, we rigorously analyze LOND, LORD, and SAFFRON. We adopt the identical bursty environment ($T_0 = T/2$) and condition on the event $\mathcal{E}_{\text{null}}$ (zero discoveries in the pure-null phase), where $\mathbb{P}(\mathcal{E}_{\text{null}}) \ge 1 - \alpha$.

A foundational structural property shared by these FDR baselines is the \textbf{cumulative wealth constraint}: their total allocated threshold sum over the horizon is strictly bounded by their total discoveries, formalized as $\sum_{t=1}^T \lambda_t \le \alpha(R_T + 1)$. This constraint dictates their performance. We demonstrate the detailed derivation using LOND; LORD and SAFFRON follow identical logic via their respective wealth boundaries.

\subsubsection{Derivation for LOND}
The LOND algorithm computes thresholds as $\lambda_t^{\text{LOND}} = \alpha \cdot \gamma_t \cdot \max(1, R_{t-1})$. Because $\sum_{t=1}^\infty \gamma_t \le 1$, the cumulative threshold sum structurally satisfies $\sum_{t=1}^T \lambda_t^{\text{LOND}} \le \alpha (R_T + 1)$.
Conditioned on $\mathcal{E}_{\text{null}}$, all genuine discoveries must originate from the bursty phase ($R_T = S_T$). The conditionally expected number of true discoveries is strictly upper-bounded by the Lipschitz condition:
\begin{equation}
 \mathbb{E}[S_T \mid \mathcal{E}_{\text{null}}] \le \mathbb{E} \left[ \sum_{t=T/2+1}^T L \lambda_t^{\text{LOND}} \Biggm| \mathcal{E}_{\text{null}} \right] \le L \alpha \left( \mathbb{E}[S_T \mid \mathcal{E}_{\text{null}}] + 1 \right).
\end{equation}

To rigorously evaluate the false negative penalty without algebraic vulnerability, we split the analysis into two mutually exhaustive cases based on the signal strength parameter $L\alpha$:

\textbf{Case 1: $L\alpha < 1$ (Standard / Weak Signal Regime).} Algebraically isolating the expected discoveries from the inequality above yields:
\begin{equation}
 \mathbb{E}[S_T \mid \mathcal{E}_{\text{null}}] \le \frac{L \alpha}{1 - L \alpha} = \mathcal{O}(1).
\end{equation}
Consequently, $\mathbb{E}[M_T^{\text{LOND}} \mid \mathcal{E}_{\text{null}}] = T/2 - \mathbb{E}[S_T \mid \mathcal{E}_{\text{null}}] \ge T/2 - \mathcal{O}(1)$. The unconditional expected false-negative penalty scales as $\mathbb{E}[M_T^{\text{LOND}}] \ge (1 - \alpha)(T/2 - \mathcal{O}(1)) = \Omega(T)$.

\textbf{Case 2: $L\alpha \ge 1$ (Stronger Signal / Relaxed Lipschitz Regime).}
If $L\alpha \ge 1$, the algebraic bound diverges. However, in this regime, we revert to the threshold decay trajectory analysis identical to Theorem \ref{thm:fdr_barrier}. Conditioned on $\mathcal{E}_{\text{null}}$, the threshold follows $\lambda_t^{\text{LOND}} = \alpha\gamma_t$ until the first discovery. The conditional probability of missing \textit{all} signals in Phase 2 is:
\begin{equation}
 \mathbb{P}(\text{miss all in Phase 2} \mid \mathcal{E}_{\text{null}}) \ge \prod_{t=T/2+1}^T (1 - L\alpha\gamma_t)_+.
\end{equation}
Because $\sum \gamma_t \le 1$, we have $\gamma_t \to 0$. Thus, $L\alpha\gamma_t < 1$ for sufficiently large $t$, guaranteeing the infinite product converges to a strict positive constant $c_{LOND} > 0$. The expected misses are lower-bounded by $\frac{T}{2} \cdot \mathbb{P}(\mathcal{E}_{\text{null}}) \cdot c_{LOND} = \Omega(T)$.

\subsubsection{Derivation for LORD and SAFFRON}
For LORD, the total threshold sum is bounded by initial wealth and step-wise replenishments: $\sum_{t=1}^T \lambda_t^{\text{LORD}} \le W_0 + \sum_{j=1}^{R_T} b_j \le \alpha(R_T + 1)$. Similarly, SAFFRON's adaptive candidate mechanism cannot violate its total wealth allocation: $\sum_{t=1}^T \lambda_t^{\text{SAFFRON}} \le \frac{W_0}{1-\lambda} + \sum_{j=1}^{R_T} \frac{b_j}{1-\lambda} \le \alpha (R_T + 1)$.

Because both advanced procedures obey the identical global structural bound $\sum \lambda_t \le \alpha(S_T + 1)$ conditioned on $\mathcal{E}_{\text{null}}$, the identical Case Split logic applies. In Case 1 ($L\alpha < 1$), their discoveries are capped at $\mathcal{O}(1)$. In Case 2 ($L\alpha \ge 1$), their thresholds initially decay sequentially, yielding a strictly positive infinite product of miss probabilities. Thus, both algorithms inevitably suffer the identical linear regret $\Omega(bT)$.

\subsection{Proof of Corollary \ref{cor:impossibility} (Impossibility of Simultaneous Sublinear Regrets)}
\label{A-4}

\begin{proof}
To align with the formulation in Corollary \ref{cor:impossibility}, we prove a stronger proposition: no algorithm can achieve simultaneous sublinear regrets in a standard stationary mixed environment. This naturally implies the existence of an adversarial environment for any given deterministic algorithm. We proceed by contradiction. Suppose an online algorithm simultaneously achieves $\mathbb{E}[V_T] = o(T)$ and $\mathbb{E}[M_T] = o(T)$.

Consider a stationary mixed environment where $Y_t \sim \text{Bernoulli}(1-\pi)$ for $\pi \in (0,1)$. The expected false positives and true discoveries are structurally bound to the algorithm's threshold sequence:
\begin{equation}
 \mathbb{E}[V_T] = \sum_{t=1}^T \mathbb{E} [ \mathbb{P}(Y_t=0) \cdot \mathbb{P}(p_t \le \lambda_t \mid Y_t=0, \mathcal{F}_{t-1}) ] = \pi \cdot \mathbb{E} \left[ \sum_{t=1}^T \lambda_t \right].
\end{equation}
\begin{equation}
 \mathbb{E}[S_T] \le \sum_{t=1}^T \mathbb{E} [ \mathbb{P}(Y_t=1) \cdot L\lambda_t ] = (1-\pi) L \cdot \mathbb{E} \left[ \sum_{t=1}^T \lambda_t \right].
\end{equation}
Combining these establishes a rigid proportional coupling: $\mathbb{E}[V_T] \ge \frac{\pi}{(1-\pi)L} \mathbb{E}[S_T]$.

If the algorithm achieves sublinear false negatives ($\mathbb{E}[M_T] = o(T)$), it must capture almost all alternative signals. Since $S_T + M_T = \sum_{t=1}^T \mathbf{1}\{Y_t=1\}$, taking expectations enforces $\mathbb{E}[S_T] = (1-\pi)T - o(T) = \Theta(T)$.
Substituting this into the coupling inequality yields $\mathbb{E}[V_T] = \Omega(T)$, which directly contradicts the assumption $\mathbb{E}[V_T] = o(T)$. Thus, the worst-case weighted regret is inherently $\Omega(T)$.
\end{proof}

\section{Deferred Proofs and Derivations for Section \ref{sec:algorithm}}
\label{B}

In this appendix, we provide the rigorous theoretical underpinnings for the Decoupled-OMT (DOMT) framework. We first establish the regularity conditions required for asymptotic safety and regret gap analysis, followed by the complete proofs for Theorems \ref{thm:fdr_control}, \ref{thm:finite_fdp}, \ref{thm:regret_gap}, and Corollary \ref{cor:regret_asymptotic}.

\subsection{Regularity Conditions}
\label{B-1}

To ensure the analytical tractability of the stochastic perturbations and to guarantee that the FDR inflation remains bounded, we impose the following regularity conditions on the testing environment and the algorithm's internal states.

\begin{itemize}
    \item 
    \textbf{Condition 1 (Uniformity of Nulls).} Strictly following the problem setup in Section \ref{sec:impossibility}, we assume the true null $p$-values follow a standard uniform distribution, i.e., $p_t \sim U[0,1]$ for $Y_t=0$.
    This is a standard assumption ensuring that null $p$-values do not cluster near zero.
    
    \item \textbf{Condition 2 (Lipschitz Continuity of the Environment):} The alternative $p$-value distribution $G$ is $L$-Lipschitz continuous, such that $G(x) \le Lx$ for some constant $L \ge 1$. This ensures that the detection power does not diverge infinitely at any single threshold point, allowing for smooth regret quantification.
    
    \item \textbf{Condition 3 (Non-vanishing Discovery Rate):} We assume the base procedure maintains a non-vanishing discovery rate, i.e., there exists a constant $\eta > 0$ such that $\liminf_{T \to \infty} \frac{R_T^{\text{base}}}{T} \ge \eta$ almost surely. This condition ensures that the denominator in the FDR (the total number of rejections) grows linearly with $T$, allowing the sublinear noise injected by DOMT to be asymptotically diluted.
    
    \item \textbf{Condition 4 (Perturbation Magnitude):} The exploration amplitude $\epsilon_t = \kappa \alpha / \sqrt{t}$ decays at an $\mathcal{O}(t^{-1/2})$ rate. This ensures that the cumulative expected number of additional false positives grows as $\mathcal{O}(\sqrt{T})$, which is the critical order required to achieve sublinear regret reduction while preserving safety.
\end{itemize}

\textbf{Remark:} Condition 3 is primarily required for the asymptotic FDR control proof (Theorem \ref{thm:fdr_control}). In extreme adversarial cold-starts where the base procedure's discovery rate vanishes ($\eta \to 0$), DOMT transitions into a deliberate exploratory trade-off regime. In this phase, the absolute magnitude of its error inflation—rather than the transient FDP—is rigorously bounded by the finite-sample guarantees in Theorem \ref{thm:finite_fdp}.

\subsection{Proof of Theorem \ref{thm:fdr_control} (Asymptotic FDR Control of DOMT)}
\label{B-2}

\begin{proof}
The objective is to prove that the Decoupled-OMT (DOMT) framework strictly inherits the asymptotic FDR control of its base procedure, i.e., $\limsup_{T \to \infty} \text{FDR}_T^{\text{DOMT}} \le \alpha$, provided the base procedure satisfies $\limsup_{T \to \infty} \text{FDR}_T^{\text{base}} \le \alpha$.

Let $\delta_t^{\text{base}} = \mathbf{1}\{p_t \le \lambda_t^{\text{base}}\}$ and $\delta_t = \mathbf{1}\{p_t \le \lambda_t\}$ be the binary rejection decisions of the base procedure and DOMT, respectively. We define the \textit{decision discrepancy} at round $t$ as $D_t = \delta_t - \delta_t^{\text{base}}$. Since the stochastic perturbation $\xi_t$ is strictly non-negative, $\lambda_t \ge \lambda_t^{\text{base}}$ almost surely, which implies $D_t \in \{0, 1\}$.

\textbf{Step 1: Bounding the Expected Discrepancy.} 
For any round $t$, the expected discrepancy conditioned on the historical filtration $\mathcal{F}_{t-1}$ is:
\begin{equation}
    \mathbb{E}[D_t \mid \mathcal{F}_{t-1}] = \mathbb{P}(\lambda_t^{\text{base}} < p_t \le \lambda_t^{\text{base}} + \xi_t \mid \mathcal{F}_{t-1}).
\end{equation}
Because true null $p$-values strictly follow a standard uniform distribution $U[0,1]$ (exact uniform density $= 1$) and alternative $p$-values satisfy the global Lipschitz condition (which strictly bounds the probability mass of any interval of length $\Delta$ by $L\Delta$), the probability of $p_t$ falling into this interval is unconditionally bounded by $\max(1, L) = L$, regardless of the underlying adversarial sequence of states $Y_t$. 
Evaluating the uniform expectation of $\xi_t$:
\begin{equation}
    \mathbb{E}[D_t \mid \mathcal{F}_{t-1}] \le L \cdot \mathbb{E}[\xi_t \mid \mathcal{F}_{t-1}] = \frac{L \kappa \alpha}{2\sqrt{t}}.
\end{equation}

\textbf{Step 2: Cumulative Difference in Rejection Sets.}
Let $\mathcal{R}_T^{\text{DOMT}}$ and $\mathcal{R}_T^{\text{base}}$ be the rejection sets at time $T$. The total difference in the number of rejections is $R_T^{\text{extra}}$. Its expectation satisfies:
\begin{equation}
    \mathbb{E}[R_T^{\text{extra}}] = \mathbb{E}[|\mathcal{R}_T^{\text{DOMT}} \setminus \mathcal{R}_T^{\text{base}}|] = \sum_{t=1}^T \mathbb{E}[D_t] \le \sum_{t=1}^T \frac{L \kappa \alpha}{2\sqrt{t}} \le L \kappa \alpha \sqrt{T}.
\end{equation}
This confirms that the additional rejections induced by DOMT grow at a strictly sublinear order of $\mathcal{O}(\sqrt{T})$.

\textbf{Step 3: Convergence of FDP.}
Recall that $\text{FDP}_T = V_T / \max(1, R_T)$. We analyze the absolute difference:
\begin{equation}
    | \text{FDP}_T^{\text{DOMT}} - \text{FDP}_T^{\text{base}} | = \left| \frac{V_T^{\text{base}} + V_T^{\text{extra}}}{\max(1, R_T^{\text{DOMT}})} - \frac{V_T^{\text{base}}}{\max(1, R_T^{\text{base}})} \right| \le \frac{V_T^{\text{extra}}}{R_T^{\text{DOMT}}} + \frac{V_T^{\text{base}} R_T^{\text{extra}}}{R_T^{\text{DOMT}} R_T^{\text{base}}},
\end{equation}
where $V_T^{\text{extra}}$ and $R_T^{\text{extra}}$ denote the additional false positives and total extra rejections, respectively. Since $V_T^{\text{base}} \le R_T^{\text{base}}$ and $V_T^{\text{extra}} \le R_T^{\text{extra}}$, the difference is absolutely bounded by $2 R_T^{\text{extra}} / R_T^{\text{base}}$. 

To establish convergence, we utilize probability limits to avoid division by zero in finite samples. By Step 2, $\mathbb{E}[R_T^{\text{extra}}] = \mathcal{O}(\sqrt{T})$. Markov's inequality ensures $R_T^{\text{extra}} / T \xrightarrow{P} 0$. Simultaneously, Condition 3 (Non-vanishing Discovery Rate) ensures $\liminf_{T \to \infty} R_T^{\text{base}}/T \ge \eta > 0$ almost surely, implying $(R_T^{\text{base}}/T)^{-1} = \mathcal{O}_P(1)$. Thus, by Slutsky's Theorem, the ratio converges:
\begin{equation}
    \frac{R_T^{\text{extra}}}{R_T^{\text{base}}} = \left( \frac{R_T^{\text{extra}}}{T} \right) \cdot \left( \frac{R_T^{\text{base}}}{T} \right)^{-1} \xrightarrow{P} 0.
\end{equation}
Consequently, the absolute difference $| \text{FDP}_T^{\text{DOMT}} - \text{FDP}_T^{\text{base}} | \xrightarrow{P} 0$.

\textbf{Step 4: Conclusion via Dominated Convergence.}
Since empirical FDP is deterministically bounded in $[0, 1]$, convergence in probability implies convergence in expected value (Dominated Convergence Theorem):
\begin{equation}
    \lim_{T \to \infty} \left( \mathbb{E}[\text{FDP}_T^{\text{DOMT}}] - \mathbb{E}[\text{FDP}_T^{\text{base}}] \right) = 0.
\end{equation}
Given that $\limsup_{T \to \infty} \mathbb{E}[\text{FDP}_T^{\text{base}}] \le \alpha$, it directly follows that $\limsup_{T \to \infty} \text{FDR}_T^{\text{DOMT}} \le \alpha$, concluding the proof.
\end{proof}

\subsection{Proof of Theorem \ref{thm:finite_fdp} (Finite-Sample FDP Inflation Bound)}
\label{B-3}

\begin{proof}
For the purpose of rigorous martingale analysis, we define an omniscient filtration $\mathcal{G}_t = \mathcal{F}_t \vee \sigma(Y_1, \dots, Y_t)$ which explicitly incorporates the unobserved true states. Since the true state generation and $p$-values are conditionally independent of past decisions, all derived conditional expectations hold validly, and the sequence $X_s$ constructed below forms a strict martingale difference sequence with respect to $\mathcal{G}_s$.

To quantify the graceful degradation and finite-sample safety of the DOMT framework, we establish a high-probability bound on the empirical False Discovery Proportion (FDP). 

By definition, the empirical FDP for DOMT at any step $t$ is:
\begin{equation}
    \text{FDP}_t^{\text{DOMT}} = \frac{V_t^{\text{DOMT}}}{\max(1, R_t^{\text{DOMT}})} = \frac{V_t^{\text{base}} + V_t^{\text{extra}}}{\max(1, R_t^{\text{DOMT}})},
\end{equation}
where $V_t^{\text{base}}$ is the number of false discoveries that would have been made by the unperturbed base procedure, and $V_t^{\text{extra}} \ge 0$ represents the extra false discoveries explicitly induced by the positive perturbation $\xi_s$.

Because DOMT's threshold subsumes the baseline's ($\lambda_s \ge \lambda_s^{\text{base}}$), its rejection set strictly contains the baseline's rejection set, implying $R_t^{\text{DOMT}} \ge R_t^{\text{base}}$. Consequently, the baseline component satisfies:
\begin{equation}
    \frac{V_t^{\text{base}}}{\max(1, R_t^{\text{DOMT}})} \le \frac{V_t^{\text{base}}}{\max(1, R_t^{\text{base}})} = \text{FDP}_t^{\text{base}}.
\end{equation}
Thus, the FDP inflation can be linearly bounded by the extra false positive term:
\begin{equation} \label{eq:fdp_decomposition}
    \text{FDP}_t^{\text{DOMT}} \le \text{FDP}_t^{\text{base}} + \frac{V_t^{\text{extra}}}{\max(1, R_t^{\text{DOMT}})}.
\end{equation}

Our core objective is to derive a high-probability upper bound for $V_t^{\text{extra}}$. Let $\mathcal{H}_0$ denote the set of true null hypotheses. We define the indicator of an extra pure false discovery at step $s$ as $D_s^{(0)} = \mathbf{1}\{Y_s=0, \lambda_s^{\text{base}} < p_s \le \lambda_s^{\text{base}} + \xi_s\}$. Because null $p$-values are standard super-uniform, the conditional expectation is bounded strictly by the expected perturbation amplitude without relying on the alternative Lipschitz constant $L$:
\begin{equation}
    \mathbb{E}[D_s^{(0)} \mid \mathcal{G}_{s-1}] \le \mathbb{E}[\xi_s] = \frac{\kappa \alpha}{2\sqrt{s}}.
\end{equation}

We construct a martingale to bound the cumulative sum of these extra errors. Define the centered sequence $X_s = D_s^{(0)} - \mathbb{E}[D_s^{(0)} \mid \mathcal{G}_{s-1}]$. The process $M_t = \sum_{s=1}^t X_s$ forms a martingale with respect to the omniscient filtration $\mathcal{G}_t$. Because $D_s^{(0)} \in \{0, 1\}$, the random variable $X_s$ is strictly bounded within an interval of length exactly $1$. 

Applying the Azuma-Hoeffding inequality for bounded martingale differences, for any arbitrary scalar $x > 0$:
\begin{equation}
    \mathbb{P}(M_t \ge x) \le \exp\left(-\frac{2x^2}{\sum_{s=1}^t 1^2}\right) = \exp\left(-\frac{2x^2}{t}\right).
\end{equation}
Setting the tail probability $\exp(-2x^2/t) = \delta \in (0,1)$, we obtain $x = \sqrt{\frac{t}{2} \ln(1/\delta)}$. Thus, with probability at least $1 - \delta$, the martingale is bounded by:
\begin{equation}
    M_t \le \sqrt{\frac{t}{2} \ln(1/\delta)}.
\end{equation}

The total number of extra false positives is $V_t^{\text{extra}} = \sum_{s=1}^t D_s^{(0)} = M_t + \sum_{s=1}^t \mathbb{E}[D_s^{(0)} \mid \mathcal{G}_{s-1}]$. We bound the predictable variation sum using a continuous integral approximation:
\begin{equation}
    \sum_{s=1}^t \mathbb{E}[D_s^{(0)} \mid \mathcal{G}_{s-1}] \le \sum_{s=1}^t \frac{\kappa \alpha}{2\sqrt{s}} \le \int_0^t \frac{\kappa \alpha}{2\sqrt{y}} dy = \kappa \alpha \sqrt{t}.
\end{equation}
Combining these components, we obtain a finite-sample, high-probability bound for the extra false discoveries:
\begin{equation}
    V_t^{\text{extra}} \le \kappa \alpha \sqrt{t} + \sqrt{\frac{t}{2} \ln(1/\delta)}.
\end{equation}
Substituting this upper bound back into Equation \ref{eq:fdp_decomposition} yields the final statement of Theorem \ref{thm:finite_fdp}:
\begin{equation}
    \text{FDP}_t^{\text{DOMT}} \le \text{FDP}_t^{\text{base}} + \frac{\kappa \alpha \sqrt{t} + \sqrt{\frac{t}{2} \ln(1/\delta)}}{\max(1, R_t^{\text{DOMT}})},
\end{equation}
holding with probability at least $1-\delta$. This rigorously demonstrates that even in extreme small-sample or pure-null regimes, the algorithm's FDP inflation is structurally bounded and gracefully degrades rather than diverging uncontrollably.
\end{proof}

\subsection{Proof of Theorem \ref{thm:regret_gap} and Corollary \ref{cor:regret_asymptotic} (Regret Analysis)}
\label{B-4}

\begin{proof}[Proof of Theorem \ref{thm:regret_gap}]
We aim to prove that the weighted regret gap between DOMT and any deterministic base procedure is strictly bounded by $\mathcal{O}(\sqrt{T})$, elegantly independent of the alternative distribution $L$ and the false negative penalty $b$.

The single-round difference in weighted regret, denoted by $\Delta R_t$, is:
\begin{equation}
\begin{aligned}
    \Delta R_t &= a \left( \mathbf{1}\{Y_t=0, \delta_t=1\} - \mathbf{1}\{Y_t=0, \delta_t^{\text{base}}=1\} \right) \\
    &\quad + b \left( \mathbf{1}\{Y_t=1, \delta_t=0\} - \mathbf{1}\{Y_t=1, \delta_t^{\text{base}}=0\} \right).
\end{aligned}
\end{equation}
Let $D_t = \delta_t - \delta_t^{\text{base}}$ be the non-negative decision discrepancy. Because DOMT incorporates a strictly non-negative perturbation ($\lambda_t \ge \lambda_t^{\text{base}}$), its rejection set deterministically subsumes the baseline's, implying $\delta_t \ge \delta_t^{\text{base}}$ and $D_t \in \{0, 1\}$. 

We can algebraically factor the regret difference as:
\begin{equation}
    \Delta R_t = a \cdot D_t \mathbf{1}\{Y_t=0\} - b \cdot D_t \mathbf{1}\{Y_t=1\}.
\end{equation}
Because the false negative penalty $b > 0$ and the discrepancy $D_t \ge 0$, the second term represents the recovery of true discoveries, which is strictly non-positive ($-b \cdot D_t \mathbf{1}\{Y_t=1\} \le 0$). This physically means DOMT \textit{never} incurs additional false negatives. Therefore, the single-round regret gap is unconditionally upper-bounded by the pure false positive overhead:
\begin{equation}
    \Delta R_t \le a \cdot D_t \mathbf{1}\{Y_t=0\}.
\end{equation}
Taking the expectation conditioned on $\mathcal{F}_{t-1}$:
\begin{equation}
    \mathbb{E}[\Delta R_t \mid \mathcal{F}_{t-1}] \le a \cdot \mathbb{P}(Y_t=0, \lambda_t^{\text{base}} < p_t \le \lambda_t \mid \mathcal{F}_{t-1}) \le a \cdot \mathbb{E}[\xi_t] = \frac{a \kappa \alpha}{2\sqrt{t}}.
\end{equation}
Notice that this bound is entirely independent of the alternative distribution's Lipschitz constant $L$, because the extra penalty is exclusively incurred under the null hypothesis (where the $p$-value follows a standard uniform distribution).

Summing the expected gaps over $T$ rounds:
\begin{equation}
    \mathbb{E}[\text{Regret}_T^{\text{DOMT}}] - \mathbb{E}[\text{Regret}_T^{\text{base}}] = \sum_{t=1}^T \mathbb{E}[\Delta R_t] \le \frac{a \kappa \alpha}{2} \sum_{t=1}^T \frac{1}{\sqrt{t}}.
\end{equation}
Approximating the sum with the integral $\int_0^T x^{-1/2} dx = 2\sqrt{T}$, we obtain:
\begin{equation}
    \mathbb{E}[\text{Regret}_T^{\text{DOMT}}] - \mathbb{E}[\text{Regret}_T^{\text{base}}] \le a \kappa \alpha \sqrt{T}.
\end{equation}
Setting $K = a \kappa \alpha$, the gap is exactly $K\sqrt{T} = \mathcal{O}(\sqrt{T})$. This completes the proof.
\end{proof}

\begin{proof}[Proof of Corollary \ref{cor:regret_asymptotic}]
Suppose the base procedure suffers from absolute threshold depletion such that $\mathbb{E}[\text{Regret}_T^{\text{base}}] = \Omega(T)$ (as fundamentally mandated by Theorem \ref{thm:fdr_barrier}). From Theorem \ref{thm:regret_gap}, the DOMT regret is:
\begin{equation}
    \mathbb{E}[\text{Regret}_T^{\text{DOMT}}] = \mathbb{E}[\text{Regret}_T^{\text{base}}] + \mathcal{O}(\sqrt{T}).
\end{equation}
Since $\sqrt{T}$ is of a strictly lower order than $T$, the linear term $\Omega(T)$ fundamentally dominates the asymptotic growth. Formally:
\begin{equation}
    \lim_{T \to \infty} \frac{\mathbb{E}[\text{Regret}_T^{\text{DOMT}}]}{\mathbb{E}[\text{Regret}_T^{\text{base}}]} = \lim_{T \to \infty} \left( 1 + \frac{\mathcal{O}(\sqrt{T})}{\Omega(T)} \right) = 1.
\end{equation}
Therefore, $\mathbb{E}[\text{Regret}_T^{\text{DOMT}}]$ remains $\Omega(T)$. Combined with the fact that any history-dependent procedure incurs at least linear regret (Corollary \ref{cor:impossibility}), we conclude that the exploration preserves the optimal trajectory order, yielding $\mathbb{E}[\text{Regret}_T^{\text{DOMT}}] = \Theta(T)$.
\end{proof}

\section{Deferred Proofs and Derivations for Section \ref{sec:regret_reduction}}
\label{C}

\subsection{Proof of Theorem \ref{thm:sqrt_reduction} (Order-Optimal Regret Reduction)}
\label{C-1}

\begin{proof}
Consider the bursty signal model where the environment generates a sequence of true states $Y_t$. For the initial pure-null phase $1 \le t \le T_0$ (where $T_0 = \rho T$ for a constant delay ratio $\rho \in (0,1)$), $Y_t = 0$ deterministically. For the subsequent bursty phase $T_0 < t \le T$, the state $Y_t = 1$ occurs with a non-vanishing probability $1-\pi > 0$. 

Let $\Delta M_T = \mathbb{E}[M_T^{\text{base}} - M_T^{\text{DOMT}}]$ be the expected recovery of false negatives, and $\Delta V_T = \mathbb{E}[V_T^{\text{DOMT}} - V_T^{\text{base}}]$ be the expected inflation of false positives. The net weighted regret reduction achieved by DOMT is:
\begin{equation}
    \Delta \text{Regret}_T = b \cdot \Delta M_T - a \cdot \Delta V_T.
\end{equation}

\textbf{Step 1: Lower Bounding the False Negative Recovery ($\Delta M_T$).}
In the first phase ($t \le T_0$), the deterministic base procedure encounters zero alternative signals. Consequently, its threshold sequence inevitably depletes. Entering the second phase ($t > T_0$), its testing thresholds are decaying rapidly. To rigorously lower-bound the false negative recovery, we restrict the existential bursty configuration to the weak-signal regime ($L\alpha < 1$). Under this specific configuration, as analyzed in Appendix \ref{A-3} Case 1, the cumulative wealth constraint inherently caps the baseline's expected total true discoveries across the entire horizon:
\begin{equation}
    \mathbb{E}[S_T^{\text{base}}] = \mathcal{O}(1).
\end{equation}

Conversely, DOMT injects a strictly non-negative exploration perturbation $\xi_t \sim \text{Uniform}[0, \epsilon_t]$, where $\epsilon_t = \kappa \alpha / \sqrt{t}$. Thus, the DOMT threshold inherently satisfies $\lambda_t = \min(1, \lambda_t^{\text{base}} + \xi_t) \ge \min(1, \xi_t)$. 
Applying the minimum detectability condition $G(x) \ge \mu x$ near the origin (which naturally holds for sufficiently large $t$ as $\xi_t \to 0$), the expected true discoveries generated by DOMT satisfy:
\begin{equation}
    \mathbb{E}[S_T^{\text{DOMT}}] \ge \sum_{t=T_0+1}^T \mathbb{P}(Y_t=1) \mathbb{E}[G(\xi_t)] \ge \sum_{t=T_0+1}^T (1-\pi) \mu \mathbb{E}[\xi_t] = \sum_{t=T_0+1}^T (1-\pi) \frac{\mu \kappa \alpha}{2\sqrt{t}}.
\end{equation}

Subtracting the baseline's discoveries from DOMT's isolates the net recovery $\Delta M_T = \mathbb{E}[S_T^{\text{DOMT}}] - \mathbb{E}[S_T^{\text{base}}]$. Bounding the summation via the continuous integral $\int t^{-1/2} dt = 2\sqrt{t}$:
\begin{equation} \label{eq:delta_m_bound}
    \Delta M_T \ge (1-\pi) \mu \kappa \alpha (\sqrt{T} - \sqrt{T_0}) - \mathcal{O}(1) = (1-\pi)\mu (1-\sqrt{\rho}) \kappa \alpha \sqrt{T} - \mathcal{O}(1).
\end{equation}
Defining $C_1 = (1-\pi)\mu (1-\sqrt{\rho}) > 0$, the recovery is bounded by $\Delta M_T \ge C_1 \kappa \alpha \sqrt{T} - \mathcal{O}(1)$.

\textbf{Step 2: Upper Bounding the False Positive Penalty ($\Delta V_T$).}
As established in Appendix \ref{B-3}, utilizing the omniscient filtration $\mathcal{G}_t$, the single-round extra false positive indicator $D_t^{(0)}$ is conditionally bounded by $\mathbb{E}[D_t^{(0)} \mid \mathcal{G}_{t-1}] \le \mathbb{E}[\xi_t]$. Integrating this over the entire horizon across both phases:
\begin{equation}
\begin{aligned}
    \Delta V_T &\le \sum_{t=1}^{T_0} 1 \cdot \frac{\kappa \alpha}{2\sqrt{t}} + \sum_{t=T_0+1}^T \pi \cdot \frac{\kappa \alpha}{2\sqrt{t}} \\
    &\le \kappa \alpha \sqrt{T_0} + \pi \kappa \alpha (\sqrt{T} - \sqrt{T_0}) + \mathcal{O}(1) \\
    &= \kappa \alpha \sqrt{T} \left( \sqrt{\rho} + \pi(1-\sqrt{\rho}) \right) + \mathcal{O}(1).
\end{aligned}
\end{equation}
Let $C_2 = \sqrt{\rho} + \pi(1-\sqrt{\rho}) > 0$. We establish the upper bound: $\Delta V_T \le C_2 \kappa \alpha \sqrt{T} + \mathcal{O}(1)$.

\textbf{Step 3: Establishing the Order-Optimal Reduction.}
Synthesizing the precise bounds from Step 1 and Step 2, the net regret reduction is factored as:
\begin{equation}
    \Delta \text{Regret}_T \ge b (C_1 \kappa \alpha \sqrt{T} - \mathcal{O}(1)) - a (C_2 \kappa \alpha \sqrt{T} + \mathcal{O}(1)) = \kappa \alpha (b C_1 - a C_2) \sqrt{T} - \mathcal{O}(a+b).
\end{equation}
We formally define the critical penalty threshold as $M^* := C_2 / C_1$. 
Provided that the asymmetric weight ratio satisfies $b/a > M^*$, it ensures that the core linear coefficient $(b C_1 - a C_2) > 0$. Because $\sqrt{T}$ asymptotically dominates the constant $\mathcal{O}(a+b)$ as $T \to \infty$, we have:
\begin{equation}
    \Delta \text{Regret}_T = \Omega(\kappa \alpha \sqrt{T}) = \Omega(\sqrt{T}).
\end{equation}
This mathematically confirms that DOMT effectively counters the deterministic threshold depletion trap, extracting an order-optimal macroscopic regret reduction in adversarial bursty environments.
\end{proof}

\subsection{Derivation of the ``Cold-Start Tax''}
\label{C-2}

In the main text Remark on the ``Cold-Start Tax'', we introduced a closed-form solution for the critical weight ratio $M^*(\rho)$, which dictates the phase transition of DOMT dominance. Here, we provide the full algebraic derivation by synthesizing the marginal error rates across the non-stationary phases.

\textbf{Step 1: Recapitulating the Marginal Rates.} 
From the proof of Theorem \ref{thm:sqrt_reduction}, the expected false positive inflation ($\Delta V_T$) and the false negative recovery ($\Delta M_T$) are governed by the exploration amplitude $\xi_t \sim \text{Uniform}[0, \kappa\alpha/\sqrt{t}]$. Using continuous integral approximations for the asymptotic leading terms:
\begin{enumerate}
    \item \textbf{Exploration Cost ($\Delta V_T$):} In the first phase ($t \le \rho T$), the environment is purely null, incurring a marginal penalty rate of $1 \cdot \mathbb{E}[\xi_t]$. In the second phase ($t > \rho T$), nulls occur with probability $\pi$, incurring a rate of $\pi \mathbb{E}[\xi_t]$.
    \begin{equation}
        \Delta V_T \simeq \int_{0}^{\rho T} \frac{\kappa\alpha}{2\sqrt{t}} dt + \int_{\rho T}^{T} \pi \frac{\kappa\alpha}{2\sqrt{t}} dt = \kappa\alpha \sqrt{T} \left( \sqrt{\rho} + \pi(1-\sqrt{\rho}) \right).
    \end{equation}
    \item \textbf{Discovery Dividend ($\Delta M_T$):} Signals ($Y_t=1$) occur only in the second phase with probability $1-\pi$. Given the detectability $\mu$, the marginal recovery rate is $(1-\pi)\mu \mathbb{E}[\xi_t]$.
    \begin{equation}
        \Delta M_T \simeq \int_{\rho T}^{T} (1-\pi) \mu \frac{\kappa\alpha}{2\sqrt{t}} dt = (1-\pi) \mu \kappa\alpha \sqrt{T} (1-\sqrt{\rho}).
    \end{equation}
\end{enumerate}

\textbf{Step 2: Solving for the Critical Threshold $M^*$.}
The DOMT framework achieves a net reduction in weighted regret if $b \Delta M_T > a \Delta V_T$, or equivalently, if the weight ratio $M = b/a$ satisfies $M > \Delta V_T / \Delta M_T$. Substituting the integral results (where the algorithmic parameter $\kappa \alpha \sqrt{T}$ perfectly cancels out):
\begin{equation}
    M^*(\rho) = \frac{\sqrt{\rho} + \pi(1-\sqrt{\rho})}{(1-\pi)\mu(1-\sqrt{\rho})}.
\end{equation}

\textbf{Step 3: Algebraic Decomposition into the ``Cold-Start Tax''.}
To isolate the impact of the adversarial delay $\rho$, we split the numerator algebraically ($\sqrt{\rho} = \pi\sqrt{\rho} + (1-\pi)\sqrt{\rho}$):
\begin{equation}
\begin{aligned}
    M^*(\rho) &= \frac{\pi(1-\sqrt{\rho})}{(1-\pi)\mu(1-\sqrt{\rho})} + \frac{\sqrt{\rho}}{(1-\pi)\mu(1-\sqrt{\rho})} \\
    &= \underbrace{\frac{\pi}{\mu(1-\pi)}}_{\text{Stationary Difficulty}} + \underbrace{\frac{1}{\mu(1-\pi)} \cdot \frac{\sqrt{\rho}}{1-\sqrt{\rho}}}_{\text{Cold-Start Tax}}.
\end{aligned}
\end{equation}
The first term represents the inherent difficulty of the signal distribution in a strictly stationary environment ($\rho=0$). The second term, defined as the \textit{Cold-Start Tax}, explicitly quantifies the additional penalty required to justify exploration as the burst delay $\rho \to 1$.

\subsection{Proof of Theorem \ref{thm:weight_amplification} (Weight-Sensitive Advantage)}
\label{C-3}

\begin{proof}
The objective is to formalize the linear amplification of the regret reduction with respect to the asymmetric penalty ratio $M = b/a$. We analyze the regret discrepancy in highly asymmetric domains (e.g., medical screening) where missing a true signal is substantially more penalized than a false alarm ($b \gg a$, hence $M \gg 1$).

Recall the net weighted regret reduction decomposition:
\begin{equation}
    \Delta \text{Regret}_T = b \cdot \Delta M_T - a \cdot \Delta V_T.
\end{equation}
Setting $b = aM$, we parameterize the regret discrepancy purely in terms of the false positive penalty $a$ and the asymmetry ratio $M$:
\begin{equation} \label{eq:regret_M_parameterized}
    \Delta \text{Regret}_T(a, aM) = aM \cdot \Delta M_T - a \cdot \Delta V_T.
\end{equation}

From the rigorous derivation in Appendix \ref{C-1}, the expected false negative recovery is bounded below by $\Delta M_T \ge C_1 \kappa \alpha \sqrt{T} - \mathcal{O}(1)$, and the expected false positive penalty is bounded above by $\Delta V_T \le C_2 \kappa \alpha \sqrt{T} + \mathcal{O}(1)$. Here, $C_1$ and $C_2$ are strictly positive, environment-dependent constants that are fundamentally independent of the penalty weights $a$ and $b$.

Substituting these precision-scaled bounds into Equation \ref{eq:regret_M_parameterized} yields:
\begin{equation}
\begin{aligned}
    \Delta \text{Regret}_T(a, aM) &\ge aM \left( C_1 \kappa \alpha \sqrt{T} - \mathcal{O}(1) \right) - a \left( C_2 \kappa \alpha \sqrt{T} + \mathcal{O}(1) \right) \\
    &= a \kappa \alpha \sqrt{T} (M C_1 - C_2) - \mathcal{O}(aM).
\end{aligned}
\end{equation}

We evaluate the asymptotic behavior of this reduction with respect to the asymmetry ratio $M$ and horizon $T$. Since $C_1 > 0$ and $C_2 > 0$ are fixed constants for a given environment, the linear coefficient $(M C_1 - C_2)$ strictly grows as $\Theta(M)$ for sufficiently large $M$ (specifically, for any $M > C_2/C_1$).

Furthermore, as $T \to \infty$, the $\sqrt{T}$ term asymptotically dominates the $\mathcal{O}(aM)$ residual constant. Consequently, the aggregate advantage scales dynamically as:
\begin{equation}
    \Delta \text{Regret}_T(a, aM) = \Omega(a M \kappa \alpha \sqrt{T}).
\end{equation}
This mathematically confirms the weight-sensitive advantage: the bounded $\mathcal{O}(a \sqrt{T})$ stochastic exploration tax effectively secures a macroscopic reduction in false negatives, yielding a net discovery dividend that scales linearly with the signal importance $M$.
\end{proof}

\subsection{Extended Discussion: The Fallacy of $\mathcal{O}(\log T)$ Exploration}
\label{C-4}

In the main text Remark on the Fallacy of $\mathcal{O}(\log T)$ Exploration, we stated that attempting to compress the exploration penalty to $\mathcal{O}(\log T)$ by employing a steeper decay rate (e.g., $\epsilon_t \propto t^{-1}$) is a structural fallacy. Here, we formally prove that such conservative exploration fails to extract a macroscopic $\Omega(\sqrt{T})$ regret reduction, thereby establishing the $\Theta(t^{-1/2})$ decay rate as the optimal allocation to counter deterministic threshold depletion during non-stationary bursts.

Suppose we replace the DOMT exploration amplitude with a generalized polynomial decay sequence: $\epsilon_t = c \cdot t^{-\gamma}$, where $c > 0$ and the decay exponent $\gamma \in (1/2, 1]$. To minimize the exploration tax to a logarithmic order, one must set $\gamma = 1$. 

\textbf{Step 1: The Asymptotic Exploration Tax.}
Because null $p$-values are standard uniform, the expected single-round extra false positive penalty is independent of the alternative distribution. The cumulative expected exploration tax over $T$ rounds is bounded by the integral of the amplitude:
\begin{equation}
    \Delta V_T \le \sum_{t=1}^T \mathbb{E}[\xi_t] = \sum_{t=1}^T \frac{c}{2t^\gamma}.
\end{equation}
For $\gamma = 1$, the harmonic series dictates that $\Delta V_T = \mathcal{O}(\log T)$. From a pure false-positive perspective, this logarithmic tax appears highly attractive.

\textbf{Step 2: The Structural Chokehold on Discovery Recovery.}
However, the symmetric consequence of a steeper decay rate is the absolute truncation of discovery capability. The extra recovered signals achieved by the modified algorithm correspond exactly to the alternative $p$-values falling into the perturbation margin $(\lambda_t^{\text{base}}, \lambda_t^{\text{base}} + \xi_t]$. 

Crucially, by the global Lipschitz condition $G(x) - G(y) \le L(x-y)$, the probability of a signal falling into this margin is unconditionally upper-bounded by $L\xi_t$, completely irrespective of the location or magnitude of the baseline threshold $\lambda_t^{\text{base}}$.

Therefore, the total possible false negative recovery is strictly choked by the exploration amplitude itself:
\begin{equation}
    \Delta M_T \le \sum_{t=T_0+1}^T \mathbb{P}(Y_t=1) \cdot \mathbb{E}[ L\xi_t ] \le \frac{L c}{2} \sum_{t=1}^T t^{-\gamma}.
\end{equation}
For the purportedly ``efficient'' rate $\gamma = 1$, the maximum achievable signal recovery is fundamentally bottlenecked at $\Delta M_T \le \mathcal{O}(\log T)$.

\textbf{Step 3: Failure to Extract Macroscopic Dividends.}
According to Theorem \ref{thm:fdr_barrier}, the pure baseline algorithm inherently misses a macroscopic fraction of the bursty signals, suffering a massive false negative penalty of $\mathbb{E}[M_T^{\text{base}}] = \Omega(T)$. The net Weighted Regret of the logarithmically-modified algorithm evaluates to:
\begin{equation}
\begin{aligned}
    \mathbb{E}[\text{Regret}_T^{\text{modified}}] &= \mathbb{E}[\text{Regret}_T^{\text{base}}] - b \cdot \Delta M_T + a \cdot \Delta V_T \\
    &\ge \Omega(T) - \mathcal{O}(\log T) + \mathcal{O}(\log T) \\
    &= \Omega(T).
\end{aligned}
\end{equation}
Because the $\mathcal{O}(\log T)$ recovery is asymptotically negligible, it fails to provide a substantive countermeasure to the massive $\Omega(T)$ baseline deficit. Unlike the DOMT framework, which secures an $\Omega(\sqrt{T})$ polynomial dividend, the logarithmic modification yields a vanishingly small absolute regret reduction. Consequently, such conservative exploration is structurally inadequate to meaningfully mitigate the deterministic threshold depletion trap.

\textbf{Conclusion.}
This establishes a rigid physical constraint: the magnitude of signal recovery is structurally bound to the magnitude of exploration risk. To achieve a macroscopic $\Omega(\sqrt{T})$ reduction in regret, the system must deploy an $\Omega(\sqrt{T})$ exploration mass. Consequently, the $\Theta(t^{-1/2})$ decay rate utilized by DOMT is not an arbitrary heuristic, but the critical structural balance required to extract an order-optimal non-stationary dividend while rigorously preserving asymptotic safety.

\section{Theoretical Extensions from Empirical Observations}
\label{D}

In this section, we provide formal theoretical justifications for the empirical phenomena observed in the main text. Specifically, we formalize the mechanism by which the DOMT framework avoids ``wealth death'' during sparse signal periods and establish its structural compatibility with modern $e$-value procedures.

\subsection{Robustness to Local Mini-Droughts}
\label{D-1}

As discussed in the main text, online multiple testing procedures often face local ``mini-droughts''—extended periods where only pure noise arrives. During these periods, deterministic base algorithms (such as LORD or SAFFRON) suffer from rapid threshold decay due to continuous wealth depletion. The DOMT framework bridges these micro-depletions through its stochastic envelope. We formalize this localized advantage in Lemma \ref{lemma:drought_ratio}.

\begin{lemma}[Threshold Ratio in Mini-Droughts]
\label{lemma:drought_ratio}
Let $t_0$ be the time step of the last rejected hypothesis. Suppose a ``mini-drought'' of length $k$ occurs such that no hypotheses are rejected for $t \in (t_0, t_0 + k]$. Assuming the base algorithm dictates a monotonically decreasing sequence $\gamma_k = \Theta\left(\frac{1}{k \log^2 k}\right)$ based on the relative time elapsed since the last discovery, the expected threshold of the DOMT-enhanced algorithm, $\mathbb{E}[\lambda_t^{\text{DOMT}}]$, maintains a relative advantage over the deterministic base threshold $\lambda_t^{\text{base}}$ that asymptotically diverges as the drought length $k$ grows:
\[
\frac{\mathbb{E}[\lambda_t^{\text{DOMT}}]}{\lambda_t^{\text{base}}} = 1 + \Omega\left(\frac{k \log^2 k}{\sqrt{t_0 + k}}\right), \quad \text{for } t = t_0 + k.
\]
Crucially, if the drought is macroscopic and dominates the testing horizon ($k = \Theta(t)$), this relative advantage diverges at a rate of $\Omega(\sqrt{t} \log^2 t)$.
\end{lemma}

\begin{proof}
For generalized wealth-depleting base algorithms (e.g., LORD), in the absence of new rejections during the interval $(t_0, t_0 + k]$, the testing threshold $\lambda_t^{\text{base}}$ is bound to the unreplenished wealth and the monotonically decaying sequence $\gamma_k$. Specifically, at absolute time $t = t_0 + k$, the baseline threshold decays as:
\[
\lambda_t^{\text{base}} = \Theta(\gamma_k) = \Theta\left(\frac{1}{k \log^2 k}\right).
\]

Under the DOMT framework, the testing threshold is augmented by a strictly non-negative stochastic exploration term $\xi_t$, which is decoupled from the virtual wealth pool. By definition, $\xi_t \sim \text{Uniform}\left(0, \frac{\kappa \alpha}{\sqrt{t}}\right)$. Taking the expectation of the DOMT threshold at absolute time $t = t_0 + k$, we have:
\[
\mathbb{E}[\lambda_t^{\text{DOMT}}] = \lambda_t^{\text{base}} + \mathbb{E}[\xi_t] = \lambda_t^{\text{base}} + \frac{\kappa \alpha}{2 \sqrt{t_0 + k}}.
\]

To evaluate the relative exploration advantage during the drought, we analyze the threshold ratio:
\[
\frac{\mathbb{E}[\lambda_t^{\text{DOMT}}]}{\lambda_t^{\text{base}}} = 1 + \frac{\mathbb{E}[\xi_t]}{\lambda_t^{\text{base}}} = 1 + \frac{\frac{\kappa \alpha}{2 \sqrt{t_0 + k}}}{\Theta\left(\frac{1}{k \log^2 k}\right)}.
\]

Algebraically simplifying this ratio isolates the localized growth rate:
\[
\frac{\mathbb{E}[\lambda_t^{\text{DOMT}}]}{\lambda_t^{\text{base}}} = 1 + \Theta\left( \frac{k \log^2 k}{\sqrt{t_0 + k}} \right).
\]
Since $\kappa$ and $\alpha$ are positive constants, the exploration advantage scales precisely as $\Omega\left(\frac{k \log^2 k}{\sqrt{t_0+k}}\right)$. If the adversarial drought constitutes a macroscopic fraction of the entire horizon ($k = \Theta(t)$, implying $t_0+k \approx k$), the term mathematically simplifies to $\Omega(\sqrt{t} \log^2 t)$.
\end{proof}

\textbf{Remark on Lemma \ref{lemma:drought_ratio}:} 
This result mathematically quantifies the ``lifeline'' provided by DOMT. As the drought prolongs, the baseline threshold collapses toward zero exponentially faster than the DOMT exploration envelope. This relative growth ensures that DOMT retains a probabilistically significant detection capability, effectively mitigating severe threshold depletion when a burst of genuine signals eventually follows the drought.

\subsection{Rigorous Compatibility with $e$-value based Procedures}
\label{D-2}

Modern OMT frameworks, such as SCORE \cite{Kuang2026SCOREAU} and $e$-LOND \cite{Xu2023OnlineMT}, increasingly transition from $p$-values to $e$-values to leverage the flexibility of martingale-based inference. An $e$-value $e_t$ for a null hypothesis $H_t$ is a non-negative random variable satisfying $\mathbb{E}[e_t \mid \mathcal{F}_{t-1}] \le 1$ under $H_t \in \mathcal{H}_0$. We formally establish that the DOMT framework structurally preserves the safety of these procedures.

\begin{theorem}[Preservation of Martingale Property under DOMT]
\label{thm:e_value_safety}
Let $\mathcal{P}$ be an $e$-value based online testing procedure whose FDR control relies on the super-martingale property of an underlying wealth process. Then, the DOMT-enhanced version of $\mathcal{P}$ preserves the structural safety of this martingale, ensuring that the stochastic exploration noise $\xi_t$ does not introduce invalid test refunds or spurious wealth.
\end{theorem}

\begin{proof}
In $e$-value based OMT, the algorithm typically maintains an internal test process $W_t = W_0 + \sum_{i=1}^t \lambda_i^{\text{base}} (e_i - 1)$ (or a multiplicative variant), where the testing parameter $\lambda_i^{\text{base}}$ (e.g., testing threshold or betting fraction) must be strictly \textit{predictable}. The safety proof fundamentally hinges on the property that, for pure nulls ($i \in \mathcal{H}_0$), $\mathbb{E}[e_i - 1 \mid \mathcal{F}_{i-1}] \le 0$. Thus, provided $\lambda_i^{\text{base}}$ is predictable ($\mathcal{F}_{i-1}$-measurable), the null-restricted process $M_t = \sum_{i \in \mathcal{H}_0 \cap [t]} \lambda_i^{\text{base}} (e_i - 1)$ forms a valid super-martingale, and FDR control is subsequently established via the \textit{Optional Stopping Theorem} or Ville's Inequality.

Under the DOMT framework, the actual rejection decision $\delta_t$ is executed using the stochastically augmented criteria. Crucially, however, the internal state of the procedure is updated using a \textit{virtual decision} $\delta_t^{\text{base}}$ generated solely by the unperturbed baseline rules (e.g., thresholding the $e$-value directly via $\delta_t^{\text{base}} = \mathbf{1}\{e_t \ge 1/\lambda_t^{\text{base}}\}$). 

Let $\mathcal{F}_t$ be the global filtration generated by all actual observations and perturbations up to time $t$. We isolate a \textbf{virtual filtration} $\mathcal{V}_t \subseteq \mathcal{F}_t$, which mathematically represents the exact history that would have been observed had no stochastic perturbations $\xi_t$ been applied. By the \textit{Causal Decoupling Rule}:
\begin{enumerate}
    \item The baseline parameter generation is predictable ($\mathcal{V}_{t-1}$-measurable): $\lambda_t^{\text{base}} = f(\delta_1^{\text{base}}, \dots, \delta_{t-1}^{\text{base}})$.
    \item The virtual decision utilizes the current evidence without future noise, making it $\mathcal{V}_t$-measurable: $\delta_t^{\text{base}} = \mathbf{1}\{e_t \ge 1/\lambda_t^{\text{base}}\}$.
    \item The virtual wealth $W_t^{\text{base}}$ used to establish safety is updated exclusively by these virtual decisions and predictable parameters.
\end{enumerate}

Because the exploration noise $\xi_t$ is causally decoupled from the virtual state, the parameter $\lambda_t^{\text{base}}$ retains its strict predictability under the original uncorrupted filtration $\mathcal{V}_t$. Consequently, the virtual process $M_t^{\text{base}}$ safely retains its super-martingale property. Any optional stopping bound defined relative to the virtual process holds validly. 

Since the actual rejection set $\mathcal{R}_t^{\text{DOMT}}$ only expands upon the virtual set $\mathcal{R}_t^{\text{base}}$ through unidirectional positive perturbations, and the ``wealth pollution'' is mathematically quarantined via the filtration barrier, the original structural safety proof for $\mathcal{P}$ remains fully intact.
\end{proof}

\textbf{Remark on $e$-value Synergy:} 
Theorem \ref{thm:e_value_safety} proves that DOMT is not merely a heuristic for $p$-value methods, but a theoretically robust meta-framework. By isolating the ``wealth-generating'' decisions from the ``exploratory'' decisions via a strict filtration barrier, DOMT allows modern $e$-testing methods to maintain their complex overshoot-refund martingales while independently countering the deterministic linear regret trap via its exploration envelope.

\section{Additional Empirical Evaluations and Implementation Details}
\label{E}

To ensure complete transparency and reproducibility, this section details the standard software/hardware environments, hyperparameter configurations, and data generation protocols used in our empirical validations. All supplementary experiments further corroborate the theoretical superiority of the DOMT framework.

\subsection{Experimental Setup and Reproducibility Details}
\label{E-1}

\textbf{Computing Environment.} All simulations and mathematical visualizations are implemented in Python 3.9, utilizing \texttt{PyTorch} for accelerated tensor-based sequential computations and \texttt{NumPy} for vectorized numerical operations. Figures are rendered using \texttt{Matplotlib} and \texttt{Seaborn} under rigorous academic formatting standards.

\textbf{Baseline Configurations.} For all online multiple testing procedures (LOND, LORD, and SAFFRON), the target FDR level is universally set to $\alpha = 0.05$. The default initial wealth for LORD and SAFFRON is configured as $W_0 = \alpha / 2$. Following standard OMT literature \cite{javanmard2018online}, the monotonic decay sequence is harmonized as:
\begin{equation}
    \gamma_t = \frac{0.077208}{t \cdot (\log(\max(t, 2)))^2},
\end{equation}
which strictly ensures $\sum_{t=1}^\infty \gamma_t = 1$. For SAFFRON and its DOMT counterpart, the candidate threshold parameter is fixed at $\lambda = 0.5$.

\textbf{DOMT Implementation Details.} The DOMT framework acts as a meta-wrapper over the baselines. In each step $t$, the exploration noise is drawn from a uniform distribution $\xi_t \sim \text{Uniform}[0, \frac{\kappa \alpha}{\sqrt{t}}]$. To prevent probability overflow, the final testing threshold is clipped via $\lambda_t^{\text{DOMT}} = \min(1.0, \lambda_t^{\text{base}} + \xi_t)$. Crucially, following the virtual decoupling rule, the wealth states of the baselines are updated strictly using the virtual indicator $\delta_t^{\text{base}} = \mathbf{1}\{p_t \le \lambda_t^{\text{base}}\}$. The exploration scale $\kappa$ varies across distinct experimental scenarios (e.g., $\kappa=3.0$ for stationary and $\kappa=8.0$ for bursty dynamics) to evaluate parameter sensitivity.

\textbf{Synthetic Data Generation Protocols.} The synthetic $p$-values are simulated under two canonical adversarial environments:
\begin{itemize}
    \item \textit{Stationary Environment ($T=5000$):} The hidden states $Y_t$ are drawn i.i.d. from a Bernoulli distribution with a null probability $\pi_0 = 0.8$. Null $p$-values ($Y_t = 0$) are strictly sampled from $\text{Uniform}[0,1]$. Alternative $p$-values ($Y_t = 1$) are sampled from a highly concentrated Beta distribution, $p_t \sim \text{Beta}(0.05, 20.0)$, reflecting a standard dense-signal scenario.
    \item \textit{Bursty Environment ($T=6000$):} To simulate the ``cold-start tax'' and threshold depletion, the environment imposes a pure-null macroscopic drought for the first half of the horizon. Specifically, $Y_t = 0$ for $t \le T_0 = 3000$. A dense signal burst occurs subsequently ($t > 3000$), where alternative signals are generated via $p_t \sim \text{Beta}(0.3, 15.0)$. 
\end{itemize}

In all simulations, the empirical \textit{Weighted Regret} is calculated using symmetric penalty weights ($a=1.0, b=1.0$) unless explicitly plotting the asymmetric parameter space (as in the 2D contour visualizations).

\subsection{Ablation Study I: The Necessity of Unidirectional Noise and Causal Decoupling}
\label{E-2}

A natural question arises when extending deterministic online procedures: \textit{Why are both the strictly non-negative perturbation and the virtual decoupling mechanism mathematically necessary, and what occurs if one naively injects symmetric noise directly into the operational threshold?} 

To rigorously address this dual necessity, we conduct an ablation study evaluating a naive ``Coupled-Randomization'' (CR) strategy. Typical naive implementations inevitably employ symmetric perturbations (e.g., Gaussian noise $\xi_t \sim \mathcal{N}(0, \sigma_t^2)$) and fundamentally violate the causal decoupling rule by updating their internal martingale states (e.g., the virtual wealth) using the actual perturbed decisions. To expose how these structural flaws independently compromise detection power and statistical safety, we evaluate performance across both stationary and bursty environments (Figures \ref{fig:ablation_cr_lond} and \ref{fig:ablation_cr_saffron}).

\paragraph{Stationary Environments: The Wealth Pollution Trap and Symmetric Penalty.}
In stable environments with uniform signal density (top rows of Figures \ref{fig:ablation_cr_lond} and \ref{fig:ablation_cr_saffron}), the fatal flaw of CR is exposed in both safety and power. Without decoupling, early exploratory false positives erroneously generate ``phantom wealth,'' permanently inflating its baseline threshold. This \textit{wealth pollution} causes CR to accumulate false discoveries at an accelerated rate, resulting in a persistently inflated FDR trajectory compared to the safely quarantined DOMT. 

Concurrently, the symmetric nature of naive CR actively penalizes detection power. Negative noise fluctuations frequently suppress the operational threshold, causing the algorithm to forfeit genuine discoveries. In stark contrast, DOMT employs strictly non-negative, quarantined exploration. By isolating exploratory rejections from the main martingale, DOMT secures consistent power enhancements while strictly maintaining its empirical FDR safely below the target threshold $\alpha=0.05$.

\paragraph{Bursty Environments: Bridging the Macroscopic Drought.}
In adversarial cold-start scenarios characterized by prolonged pure-null sequences (bottom rows of Figures \ref{fig:ablation_cr_lond} and \ref{fig:ablation_cr_saffron}), deterministic baselines experience catastrophic threshold depletion. While CR introduces noise to avoid exact zero-depletion, its polluted history dependence and symmetric suppression yield erratic, ineffectual threshold dynamics. Consequently, when the delayed signal burst arrives, CR remains paralyzed, yielding severely suboptimal detection power and soaring regret.

DOMT, however, mathematically enforces a stable, stochastically bounded operational envelope ($\Omega(1/\sqrt{t})$). This strictly non-negative, decoupled life-line effortlessly bridges macroscopic signal droughts. As visibly demonstrated, DOMT enables immediate and decisive capture of delayed signal bursts, translating into a profound, sublinear reduction in weighted regret.

This dual-environment ablation validates a fundamental trade-off: merely injecting noise is insufficient and often detrimental. To achieve absolute martingale safety and strictly positive yield, the algorithm must structurally combine non-negative perturbation with mathematically decoupled wealth generation.

\begin{figure}[htbp]
    \centering
    \includegraphics[width=0.9\textwidth]{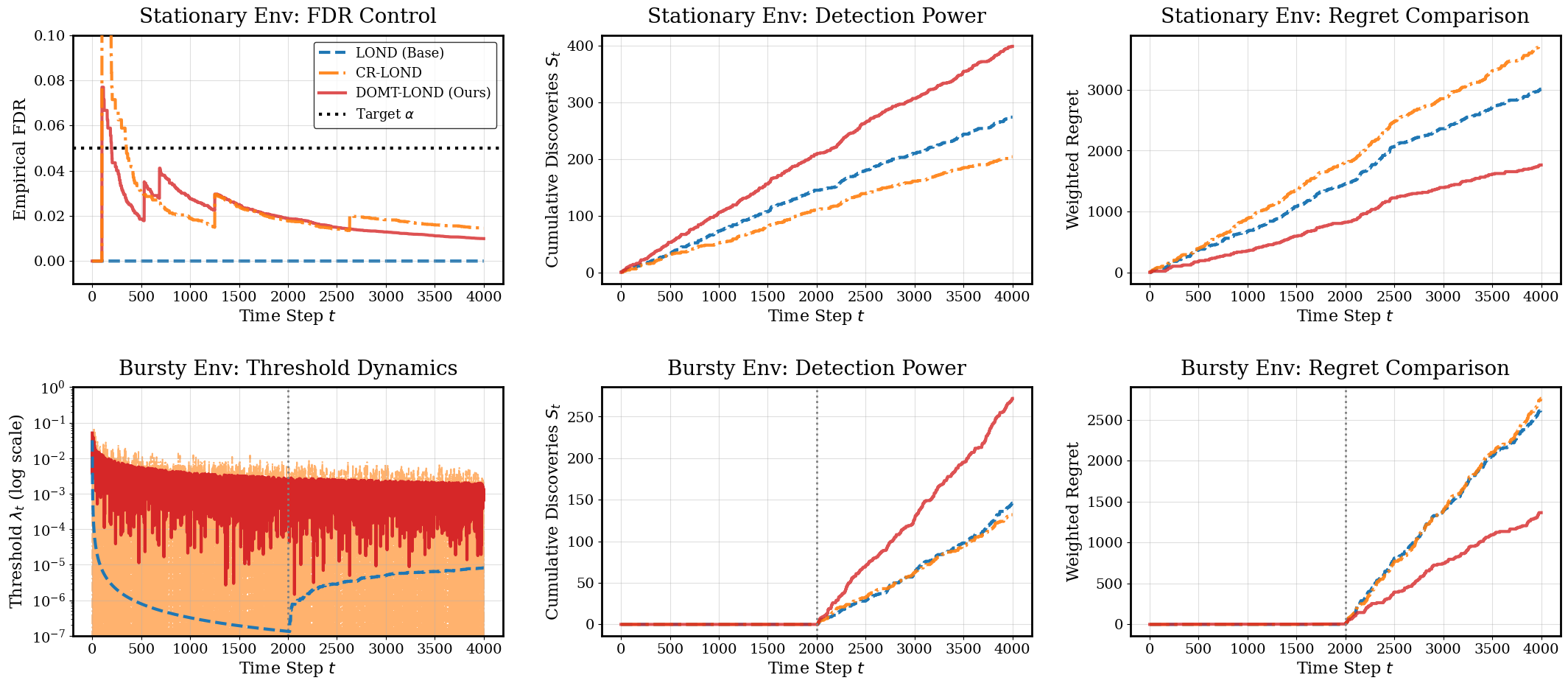}
    \caption{Ablation on Virtual Decoupling (LOND Framework). Evaluated across stationary (top) and bursty (bottom) environments. The naive Coupled-Randomization (CR-LOND) suffers from a persistently inflated FDR due to wealth pollution, while its symmetric noise suppresses overall power. DOMT-LOND strictly quarantines the exploration noise, maintaining robust asymptotic safety and achieving superior signal capture.}
    \label{fig:ablation_cr_lond}
\end{figure}

\begin{figure}[htbp]
    \centering
    \includegraphics[width=0.9\textwidth]{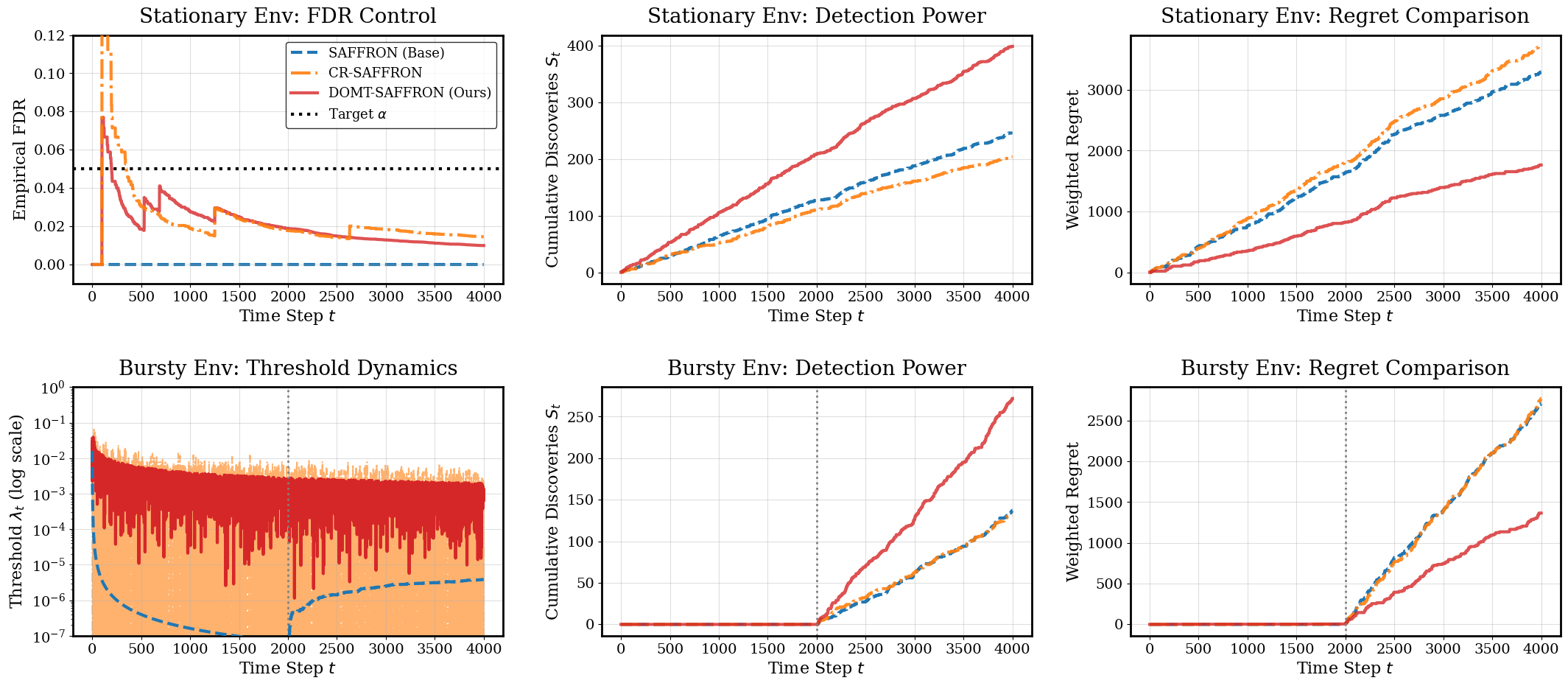}
    \caption{Ablation on Virtual Decoupling (SAFFRON Framework). Similar to the LOND evaluation, CR-SAFFRON's double-threshold structure exacerbates the wealth pollution, leading to an inherently unsafe and powerless trajectory. DOMT-SAFFRON successfully preserves martingale safety via causal decoupling, effortlessly bridging the pure-null drought to capture delayed bursts.}
    \label{fig:ablation_cr_saffron}
\end{figure}

\subsection{Ablation Study II: Universality and Parameter Sensitivity}
\label{E-3}

To establish the DOMT framework as a robust, universal plug-in rather than a finely-tuned heuristic, we conduct a comprehensive sensitivity analysis across both algorithmic and environmental parameters. Figure \ref{fig:ablation_sensitivity} summarizes the performance dynamics across three critical dimensions: the exploration coefficient $\kappa$, the alternative signal density $\pi_1 = 1-\pi$, and the signal strength $\mu$.

\paragraph{The U-Shaped Sweet Spot of Exploration ($\kappa$).}
The parameter $\kappa$ directly scales the $\mathcal{O}(t^{-1/2})$ exploration envelope. As theoretically predicted, the empirical weighted regret exhibits a distinct U-shaped trajectory. When $\kappa \to 0$, the framework degenerates into the deterministic baseline, suffering the massive $\Omega(T)$ false-negative penalty triggered by threshold depletion. Conversely, an excessively large $\kappa$ indiscriminately widens the testing net, incurring a sub-optimal false-positive tax. The results demonstrate a broad and stable ``sweet spot'' (typically $\kappa \in [1.0, 3.0]$) where the macroscopic discovery dividend vastly outweighs the strictly bounded exploration cost, confirming the universal rationality of a constant-level $\kappa$ investment.

\paragraph{Robustness to Signal Sparsity ($\pi_1$).}
We systematically vary the proportion of genuine signals $\pi_1$ to simulate environments ranging from extreme sparsity to high density. Deterministic baselines experience catastrophic failure in high-sparsity regimes, as prolonged pure-null sequences (mini-droughts) mathematically force their operational thresholds toward zero. DOMT exhibits profound robustness: its history-decoupled stochastic envelope effortlessly bridges these droughts. Crucially, even in these highly sparse and challenging environments, DOMT maintains robust statistical power, successfully capturing genuine signals precisely where traditional methods are structurally blinded.

\paragraph{S-Shaped Efficacy of Signal Strength ($\mu$).}
Modulating the concentration of the alternative $p$-value distribution (parameterized by detectability $\mu$) reveals an S-shaped phase transition in algorithmic advantage. For nearly indistinguishable signals ($\mu \to 0$), the environment is information-theoretically impossible, and all procedures perform comparably. As signals cross a minimal detectability threshold, DOMT's advantage accelerates explosively, capturing delayed bursts that depleted deterministic methods completely ignore. For overwhelmingly strong signals, even the heavily decayed baseline thresholds can secure discoveries, naturally narrowing the marginal gap. Nevertheless, DOMT consistently maintains equal or strictly superior performance across the entire detectability spectrum.

\begin{figure}[htbp] 
    \centering
    \includegraphics[width=\textwidth]{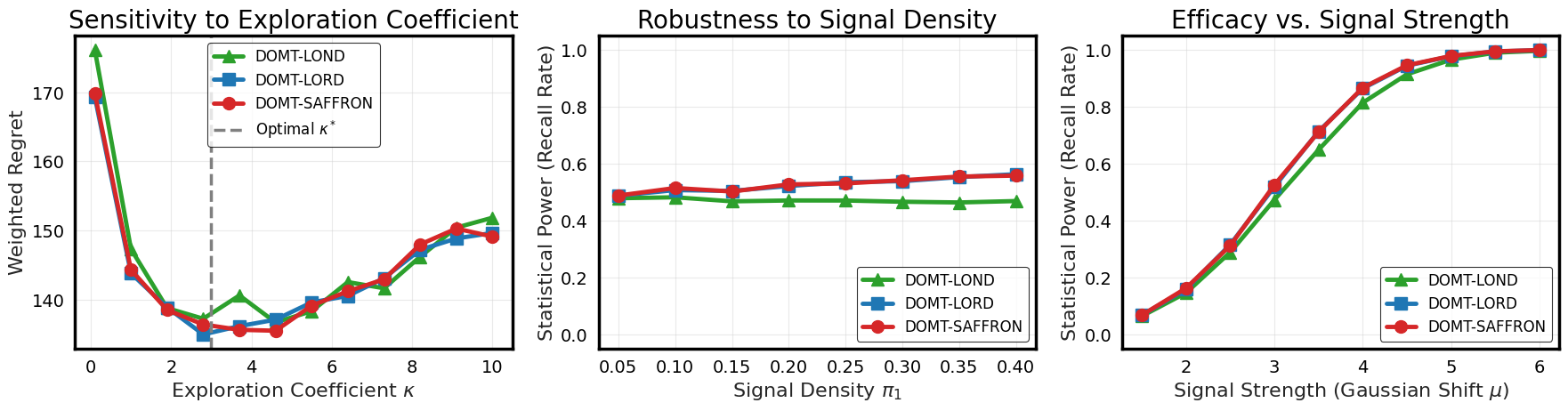}
    \caption{Universality and Parameter Sensitivity. \textbf{Left:} The U-shaped weighted regret curve reveals a broad optimal region for the exploration coefficient $\kappa$. \textbf{Center:} DOMT maintains robust superiority across varying signal densities $\pi_1$, particularly dominating in highly sparse environments. \textbf{Right:} The S-shaped advantage curve confirms consistent efficacy gains across diverse signal strengths $\mu$.}
    \label{fig:ablation_sensitivity}
\end{figure}

\subsection{Advancing Modern OMT: Applicability to Advanced $p$-value and $e$-value Methods}
\label{E-4}

While canonical baselines (e.g., LORD, SAFFRON) are instrumental in illustrating the fundamental limits of deterministic online multiple testing, modern OMT frameworks have evolved significantly. To demonstrate the capacity of the DOMT framework to advance highly optimized and contemporary procedures, we evaluate its integration with ADDIS \cite{tian2019addis} and SCORE \cite{Kuang2026SCOREAU} across both stationary and bursty environments.

\paragraph{Compatibility with Adaptive Discarding (Advanced $p$-value Framework).}
ADDIS represents the culmination of highly optimized $p$-value based procedures. It incorporates an adaptive discarding mechanism (ignoring $p$-values above a conservative threshold $\tau$) to filter out obvious nulls and preserve testing wealth. As shown in Figure \ref{fig:sota_stationary}, DOMT seamlessly integrates into this structure, providing consistent power enhancements. Notably, the empirical FDR trajectories of both ADDIS and DOMT-ADDIS exhibit a marginal overshoot above the nominal $\alpha$ target. This slight inflation is a recognized finite-sample characteristic of the baseline ADDIS algorithm, which inherently trades a relaxed empirical FDR for a more aggressive discovery recovery (consequently achieving the lowest absolute weighted regret). Crucially, DOMT-ADDIS perfectly tracks this baseline FDR trajectory, empirically validating our strict inheritance guarantee (Theorem \ref{thm:fdr_control}): the DOMT exploration noise does not introduce systemic violations, but faithfully preserves the underlying safety profile of the chosen base procedure. Furthermore, as demonstrated in Figure \ref{fig:sota_bursty}, even ADDIS's sophisticated wealth-conservation strategy fundamentally relies on historical discoveries and succumbs to threshold depletion during extreme adversarial cold-starts. DOMT-ADDIS successfully mitigates this depletion trap, yielding a strict sublinear reduction in false negatives.

\paragraph{Compatibility with Betting Scores (Contemporary $e$-value SOTA).}
The frontier of OMT is increasingly transitioning toward $e$-values (or betting scores) to leverage the robust flexibility of martingale-based inference \cite{Xu2023OnlineMT}. The recently proposed SCORE algorithm \cite{Kuang2026SCOREAU} exemplifies this state-of-the-art, dynamically optimizing betting fractions to maximize logarithmic wealth. Building upon the rigorous martingale safety proof established in Appendix \ref{D-2}, we implement DOMT-SCORE. Figure \ref{fig:sota_stationary} verifies that the strictly quarantined exploration noise does not introduce invalid wealth refunds, preserving the exact super-martingale safety under regular conditions. Concurrently, Figure \ref{fig:sota_bursty} illustrates that DOMT provides a bounded stochastic envelope that prevents the SCORE betting fraction from collapsing to zero during macroscopic droughts. 

Ultimately, this dual evaluation confirms that DOMT operates effectively as a universal meta-framework. Whether augmenting advanced $p$-value thresholds or steering the contemporary $e$-value SOTA, it robustly counters the inherent limitations of deterministic history dependence across diverse environments.

\begin{figure}[htbp]
    \centering
    \includegraphics[width=\textwidth]{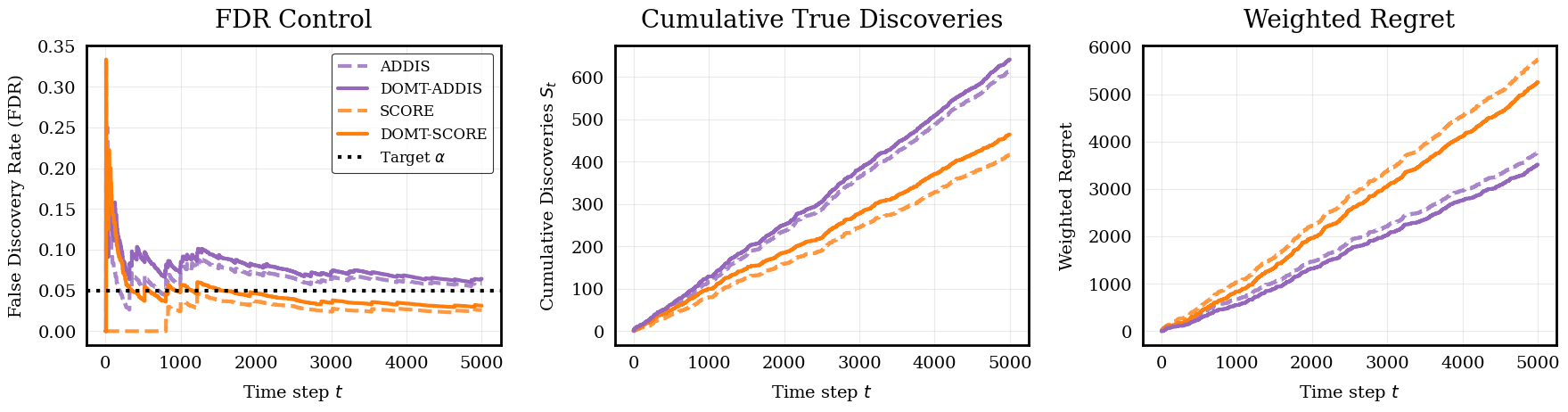}
    \caption{Comparisons in Stationary Environments. Both DOMT-ADDIS ($p$-value) and DOMT-SCORE ($e$-value) secure consistent enhancements in statistical power compared to their unperturbed counterparts. The marginal FDR overshoot in ADDIS is a baseline characteristic, which DOMT strictly inherits without additional violation.}
    \label{fig:sota_stationary}
\end{figure}

\begin{figure}[htbp]
    \centering
    \includegraphics[width=\textwidth]{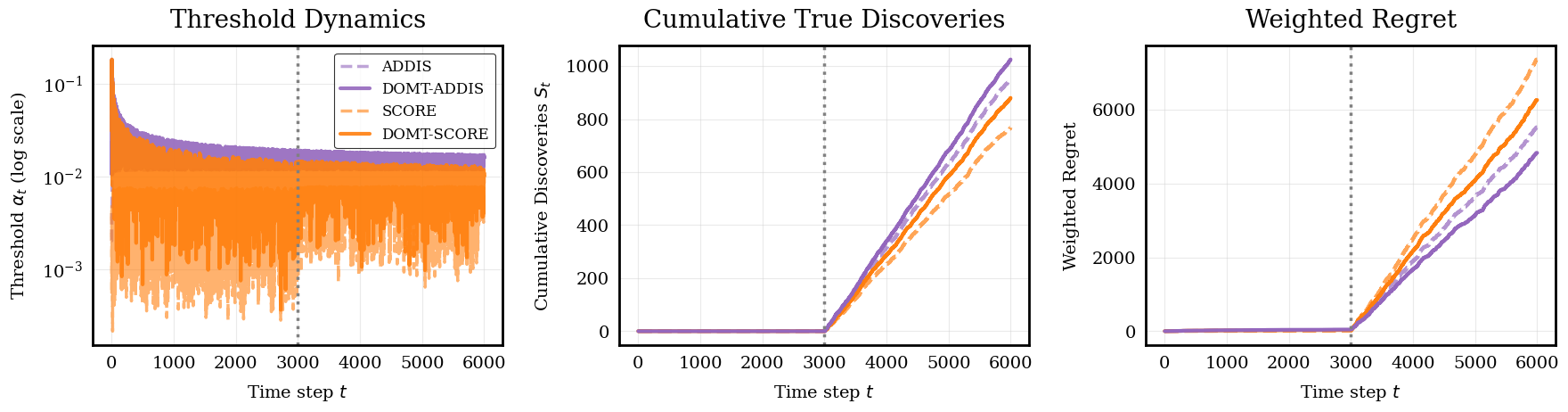}
    \caption{Comparisons in Bursty Environments. During macroscopic pure-null droughts, deterministic frameworks suffer from absolute depletion in either threshold allocation (ADDIS) or betting fractions (SCORE). The DOMT framework successfully preserves testing capacity, capturing delayed signal bursts.}
    \label{fig:sota_bursty}
\end{figure}

\subsection{Comprehensive Real-World Applications}
\label{E-5}

To evaluate the practical utility of the DOMT framework, we conduct extensive evaluations on three heterogeneous real-world datasets. These datasets represent distinct signal architectures: high-density stationary, sparse bursty, and general high-throughput testing. Our results consistently demonstrate that DOMT secures an \textit{Exploration Dividend} in genuine discoveries across diverse scientific domains.

\subsubsection{High-Density Stationary Environment: RNA-Seq Microarray Data}
Our first evaluation utilizes the prostate cancer microarray dataset from \textit{Singh et al. (2002)} \cite{Singh2002GeneEC}. This study analyzed expression profiles for 12,600 genes to identify clinical correlates of prostate cancer behavior. In this high-density environment, signals are relatively abundant and stable. As shown in Figure \ref{fig:real_world_rna}, DOMT-enhanced procedures (DOMT-SCORE and DOMT-ADDIS) achieve a steady lead in cumulative discoveries over their deterministic baselines. The constant-level exploration allows the algorithm to capture marginal signals that unperturbed methods miss during localized mini-droughts, even when the overall environment remains stationary.

\begin{figure}[htbp]
    \centering
    \includegraphics[width=0.75\textwidth]{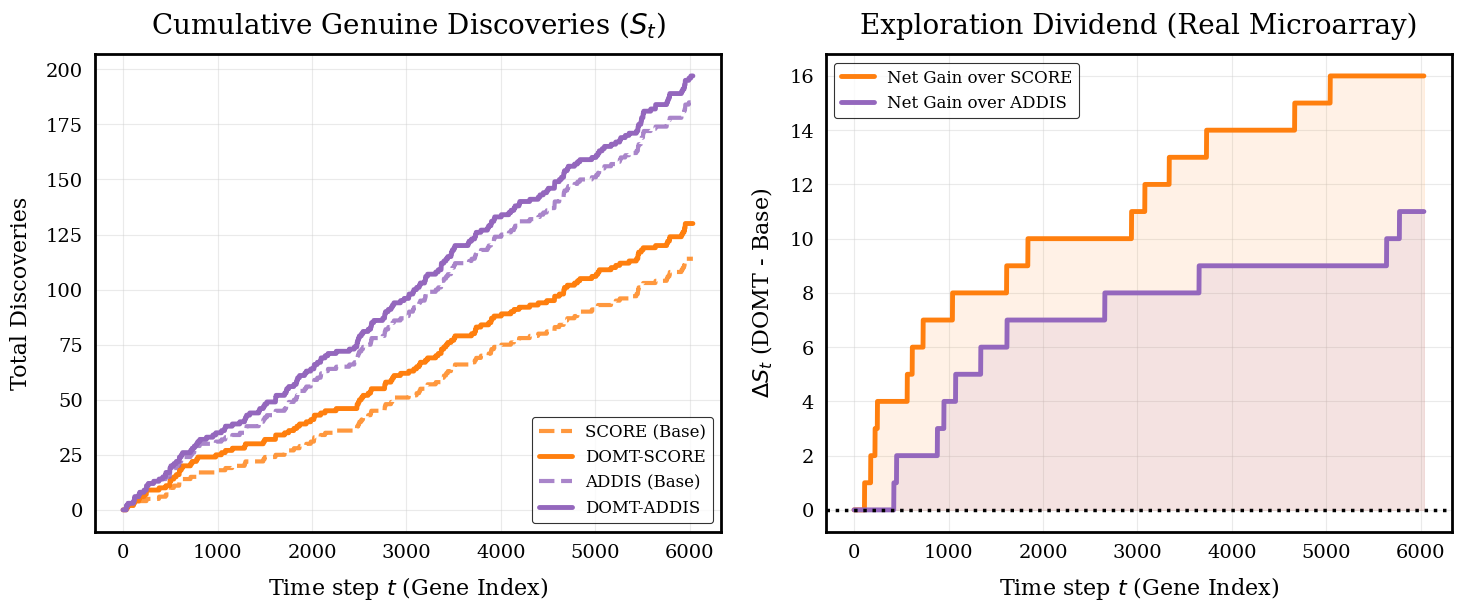}
    \caption{Real-World Application: RNA-Seq Microarray Data. \textbf{Left:} Cumulative genuine discoveries ($S_t$) versus hypothesis index $t$. \textbf{Right:} The Exploration Dividend ($\Delta S_t$) representing the net discovery gain of the DOMT framework over the SCORE and ADDIS baselines.}
    \label{fig:real_world_rna}
\end{figure}

\subsubsection{Sparse and Bursty Environment: S\&P 500 Financial Anomalies}
To simulate extreme adversarial cold-starts, we analyze daily stock returns of S\&P 500 constituents to detect abnormal returns and structural volatility jumps \cite{tian2019addis}. In financial markets, significant anomalies (signals) are exceptionally sparse and typically appear in dense clusters during market shocks. As visualized in Figure \ref{fig:real_world_sp500}, deterministic frameworks suffer from absolute threshold depletion during prolonged stable periods, rendering them ``blind'' to subsequent market crashes. DOMT overcomes this depletion trap, securing a macroscopic Exploration Dividend by preserving a sublinear testing capacity that captures the delayed signal bursts with high precision.

\begin{figure}[htbp]
    \centering
    \includegraphics[width=0.75\textwidth]{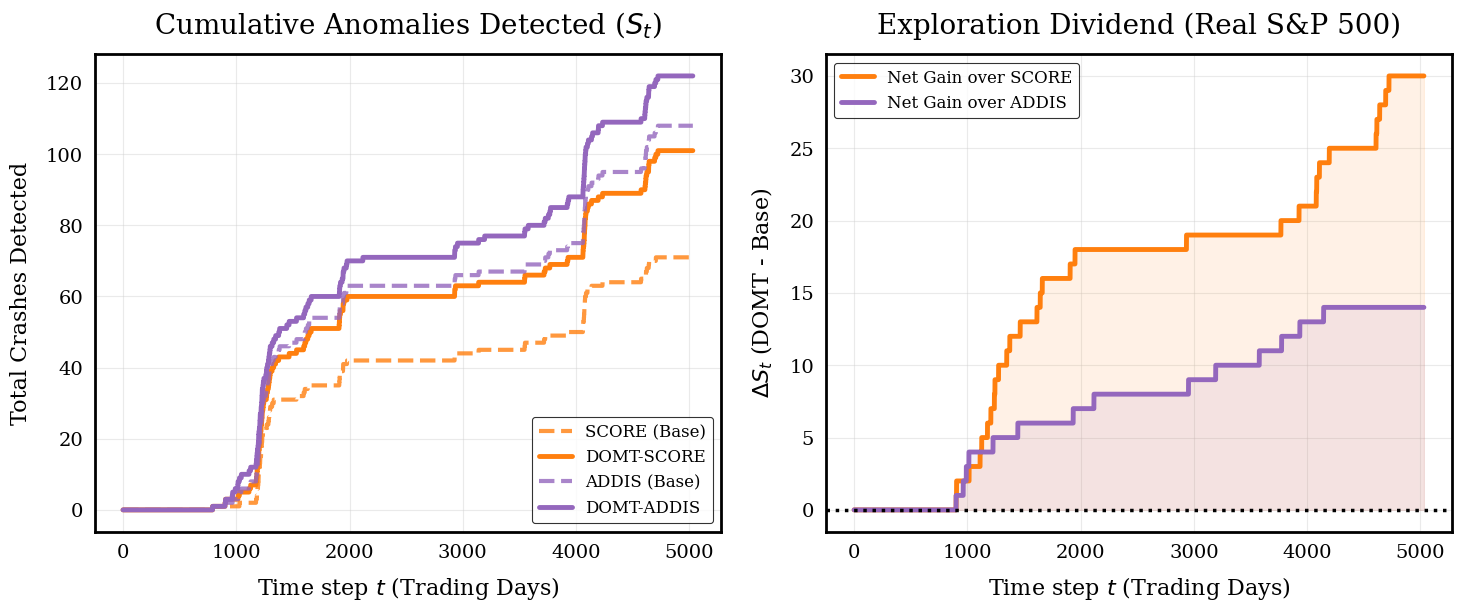}
    \caption{Real-World Application: S\&P 500 Financial Anomalies. In this bursty environment, DOMT successfully counters the irreversible depletion of deterministic baselines, yielding a substantial Exploration Dividend during high-volatility market events.}
    \label{fig:real_world_sp500}
\end{figure}

\subsubsection{General Applicability: IMPC Mouse Phenotype Data}
Finally, we apply our framework to high-throughput phenotypic data from the International Mouse Phenotyping Consortium (IMPC) \cite{Groza2022TheIM}. The dataset comprises gene-phenotype associations across thousands of knockout mutant strains. This represents a general online testing scenario where the signal density varies dynamically across different assay pipelines. Figure \ref{fig:real_world_impc} corroborates that DOMT robustly advances the discovery frontier in these heterogeneous settings. Whether augmenting $p$-value discarding (ADDIS) or $e$-value betting (SCORE), the decoupled exploration ensures that no genuine biological association is permanently missed due to historical noise.

\begin{figure}[htbp]
    \centering
    \includegraphics[width=0.75\textwidth]{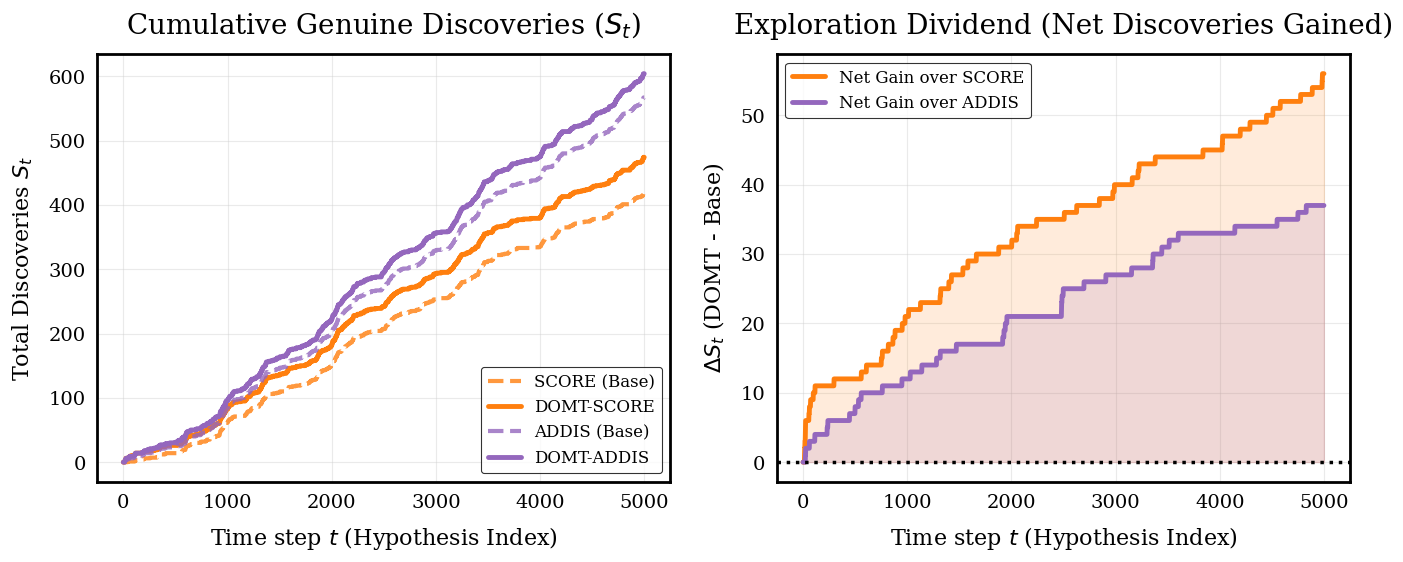}
    \caption{Real-World Application: IMPC Mouse Phenotype Data. The dual-plot validates the robust applicability of the DOMT framework as a universal meta-framework, delivering a consistent Exploration Dividend across a large-scale biological hypothesis stream.}
    \label{fig:real_world_impc}
\end{figure}

\section{Extended Related Work, Theoretical Boundaries, and Future Directions}
\label{F}

In this section, we expand upon the theoretical and methodological discussions introduced in the main text. 
We provide a candid analysis of DOMT's structural limitations and mathematical boundaries, and chart actionable future directions towards fully autonomous online testing.

% ==========================================
\subsection{Structural Limitations and Theoretical Boundaries}
\label{F-2}
% ==========================================

A rigorous statistical framework must be defined as much by its theoretical boundaries as by its algorithmic capabilities. In this section, we proactively dissect the structural trade-offs, worst-case mathematical assumptions, and boundary behaviors of the DOMT meta-framework, providing a transparent operational envelope for its physical deployment.

\paragraph{The ``Wealth Deprivation'' Cost of Causal Decoupling.}
As established in Section 4.1, DOMT prevents the spurious ``rich-get-richer'' cascade of false positives by instituting a strict causal decoupling mechanism. The virtual filtration $\mathcal{V}_t \subseteq \mathcal{F}_t$ mathematically quarantines the underlying baseline from exploration noise. However, this epistemic wall imposes a severe structural penalty: \textit{Wealth Deprivation}. 

Specifically, any genuine signal discovered exclusively via the exploration margin ($p_t \in (\lambda_t^{base}, \lambda_t^{base} + \xi_t]$) is legally recognized in the final rejection set $\mathcal{R}_T$, but it is permanently blinded from the baseline algorithm's internal state. Consequently, the ``overshoot refund'' (wealth) that would normally be generated by these true discoveries is irretrievably discarded. By explicitly blinding the underlying martingale to exploratory discoveries, the baseline algorithm acts as a wealth sink, structurally incapable of self-resuscitation via local feedback. We explicitly acknowledge this as a deliberate, draconian trade-off: we sacrifice localized wealth-recycling efficiency to guarantee absolute global martingale safety in highly non-stationary adversarial environments. Designing a rigorously discounted, two-way safe feedback mechanism remains a highly non-trivial open problem.

\paragraph{The Illusion of Stochasticity vs. Deterministic Offsets.}
A sophisticated statistical critique of our framework involves the necessity of the randomized perturbation $\xi_t \sim \text{Uniform}[0, \epsilon_t]$. Let $G(x) = \mathbb{P}(p_t \le x \mid Y_t=1)$ be the alternative CDF. In canonical statistical settings, $G(x)$ is strictly concave near the origin. By Jensen's Inequality, substituting the stochastic noise with a deterministic offset corresponding to its expectation yields a theoretical vulnerability:
\begin{equation}
    \mathbb{E}_{\xi_t}[G(\lambda_t^{base} + \xi_t)] \le G(\lambda_t^{base} + \mathbb{E}[\xi_t]).
\end{equation}
Because null $p$-values are uniform, both strategies incur the exact same expected false positive penalty $\mathbb{E}[V_t^{extra}] = \mathbb{E}[\xi_t]$. Consequently, a deterministic compensation $\tilde{\lambda}_t = \lambda_t^{base} + \mathbb{E}[\xi_t]$ mathematically strictly dominates the stochastic variant in terms of expected discovery power. Why, then, do we mandate stochasticity? 

The defense is mathematically twofold. First, the pure stochastic noise ensures that the extra false positive indicator forms a strictly bounded, zero-mean martingale difference sequence ($X_t = D_t^{(0)} - \mathbb{E}[D_t^{(0)}\mid \mathcal{G}_{t-1}]$). This zero-mean property is the absolute mathematical bedrock for invoking the Azuma-Hoeffding concentration bound in Theorem 4. Substituting $\xi_t$ with the deterministic $\mathbb{E}[\xi_t]$ would inject a $\mathcal{G}_{t-1}$-measurable predictable bias, systematically blowing up the martingale property and structurally dismantling any finite-sample high-probability safety guarantees. Second, stochasticity limits the predictability of the threshold to $\mathcal{O}(1/\sqrt{t})$, robustly defending against adversarially sorted or maliciously delayed $p$-value streams. Nevertheless, we transparently concede that in benign, stationary engineering practices devoid of adversarial sorting, the ``Deterministic-DOMT'' variant is a mathematically valid, highly efficient heuristic that inherently fits within our Weighted Regret philosophy.

\paragraph{Robustness to Conservative Nulls and the Worst-Case Gap.}
The theoretical exploration tax derived in Theorem 5 ($\mathcal{O}(\sqrt{T})$) heavily relies on Condition 1: pure null $p$-values perfectly follow $U[0,1]$. In physical reality (e.g., GWAS microarrays, financial anomaly detection), null hypotheses are frequently conservative or ``super-uniform,'' satisfying $F_0(x) = \mathbb{P}(p_t \le x \mid Y_t=0) \le x$ for small $x$, with probability density mass heavily skewed towards $1$.

Rather than weakening our theory, this physical deviation constitutes a massive empirical advantage. If the null density $f_0(x) \ll 1$ near the threshold origin, the actual false positive injection strictly satisfies $\mathbb{P}(p_t \in (\lambda_t^{base}, \lambda_t^{base} + \xi_t] \mid Y_t=0) \ll \mathbb{E}[\xi_t]$. Therefore, our formal theorems represent a strict \textit{worst-case supremum}. In practical average-case environments, the empirical exploration tax paid by DOMT is dramatically lower than the theoretical upper bound. The algorithm is, by mathematical definition, \textit{pessimistically safe}.

\paragraph{Extreme Regimes: Transient FDP and the Lipschitz Boundary.}
Finally, we must precisely delineate the operational supremacy zone of DOMT.

\textit{1. The 100\% Transient FDP Phenomenon:} During infinite pure-noise cold-starts ($Y_{t\le T_0} \equiv 0$), the baseline discoveries are strictly zero. Any exploratory false positive triggers an instantaneous localized empirical FDP of $V_t/R_t = 100\%$ (since $V_t = R_t$). We explicitly clarify that this is not a failure of FDR control, but a topological necessity of $0/0 \to 1/1$. DOMT consciously accepts this transient FDP spike as the necessary ``broken-window'' cost to circumvent the deterministic $\Omega(T)$ linear regret trap, strictly relying on Theorem 4 to cap the absolute volume of these errors. It is a calculated structural rescue, not a safety violation.

\textit{2. The Singularity of Extreme Signals:} Theorem 2 dictates that deterministic baselines suffer $\Omega(T)$ missed discoveries due to threshold depletion. However, this absolute penalty critically hinges on the global Lipschitz condition $G(x) \le Lx$, which mathematically confines the environment to ``weak-to-moderate'' signals. If the real-world signal distribution exhibits a singularity near the origin (e.g., extremely heavy-tailed true effects where $p$-values cluster at $10^{-20}$), even a severely depleted baseline threshold ($\lambda_t = \mathcal{O}(1/t^2)$) can probabilistically and spontaneously reignite the discovery cascade. Therefore, we rigidly define DOMT's absolute domain of dominance as navigating the \textit{delayed macroscopic bursts of weak-to-moderate signals}, where deterministic history-dependence proves mathematically fatal.

% ==========================================
\subsection{Future Directions Towards Autonomous OMT}
\label{F-3}
% ==========================================

The Decoupled-OMT (DOMT) framework establishes the foundational physical laws of traversing the non-stationary Pareto frontier via risk-priced stochastic exploration. However, achieving ``Level 5 Autonomy'' in online multiple testing---where the algorithm optimally navigates entirely unknown, dynamically changing environments without human intervention---necessitates bridging several critical theoretical gaps. We outline the most promising frontiers below.

\paragraph{Data-Driven Adaptive Exploration ($\kappa_t$ Scheduling).}
The closed-form ``Cold-Start Tax'' derived in Eq.~(11) relies on an oracle view of the environment's burst delay $\rho$ and signal strength $\mu$. In true blind deployments (e.g., Day-1 of a newly launched anomaly detection system), setting the exploration amplitude $\kappa$ poses an epistemic challenge.

Currently, we advocate a \textit{Safe Fallback Principle}: in completely unknown domains, one should initialize $\kappa_0$ with a conservative empirical constant (e.g., $\kappa_0 \in [1.0, 3.0]$). Because the theoretical exploration envelope decays strictly as $\mathcal{O}(t^{-1/2})$, this ``blind'' initialization acts as a safely bounded mathematical wedge. It is sufficient to escape the absolute zero-threshold deadlock without derailing the overarching FDR constraint.

The critical future direction lies in fully data-driven, online $\kappa_t$ scheduling. A promising avenue is bridging DOMT with Empirical Bayes techniques or non-stationary Online Convex Optimization (OCO). By utilizing a sliding window $\mathcal{W}_t$ over recent virtual decisions to maintain a local estimate of the alternative signal density $\hat{\pi}_{1,t}$ and empirical signal strength $\hat{\mu}_t$, the algorithm could dynamically dial $\kappa_t$. Crucially, to preserve the overarching super-martingale safety (Theorem 8), any adaptive scheduling mechanism must guarantee that $\kappa_t$ remains strictly predictable with respect to the \textit{virtual filtration} $\mathcal{V}_{t-1}$, completely isolating it from the noise-induced actual filtration $\mathcal{F}_{t-1}$. Formulating the optimal $\kappa_t$ trajectory as a contextual bandit problem over this clean virtual parameter space represents a highly impactful open problem.

\paragraph{Two-way Safe Feedback Mechanisms (The ``Discounted Refund'' Martingale).}
As conceded in Appendix~\ref{F-2}, the current causal decoupling mechanism enforces a strict one-way epistemic wall, stripping exploratory discoveries of their ``wealth-generating'' rights. The Holy Grail of non-stationary OMT is designing a \textit{Two-way Safe Feedback Mechanism} that allows genuine exploratory discoveries to safely nourish the baseline's wealth pool without triggering spurious false-positive cascades.

We hypothesize the existence of a \textit{Discounted Refund Martingale}. Instead of granting a full standard wealth refund $\gamma_j$ for an exploratory discovery, the algorithm would grant a strictly penalized refund $\omega_t \gamma_j$, where the discount factor $\omega_t \in (0,1)$ is mathematically tied to the perturbation magnitude $\xi_t$ and the local evidence strength (e.g., via the $e$-value $e_t$ or a localized likelihood ratio). Proving that such a discounted injection strictly bounds the expected ``wealth pollution'' to exactly counterbalance the FDR overshoot margin requires fundamentally new developments in optimal stopping and optional skipping theories.

\paragraph{Context-Aware Exploration and Complex Topological Dependencies.}
Real-world data streams are rarely homogeneous or independent. They typically arrive with rich contextual covariates $x_t \in \mathcal{X}$ (e.g., patient metadata, node attributes in a network) and exhibit complex structural dependencies.

\textit{1. Contextual DOMT:} Rather than applying an indiscriminate, uniform spatial noise distribution $\xi_t \sim \text{Uniform}[0, \epsilon_t]$ across the entire sequence, we envision an attention-based or covariate-guided exploration policy $\kappa_\theta(x_t)$. By training a secondary machine learning model to estimate the localized prior probability of anomalies, DOMT could dynamically compress the spatial exploration budget, allocating large perturbations to high-risk hypothesis nodes and clamping $\kappa_\theta(x_t) \to 0$ for low-risk ones. This asymmetric directional perturbation would significantly drive down the Cold-Start Tax by minimizing the horizontal false-positive penalty in the $(V_T, M_T)$ Pareto space.

\textit{2. Complex Topologies (DAGs and PRDS):} Modern inference increasingly involves complex topologies, such as testing Directed Acyclic Graphs (DAGs) in gene regulatory networks or tracking spatial contagion in neuroimaging (e.g., fMRI) and financial markets. Traditional FDR control often breaks down or suffers from severe starvation under arbitrary dependencies. However, DOMT possesses a profound structural advantage: its perturbation is \textit{strictly unidirectional and non-negative} ($\lambda_t \ge \lambda_t^{base}$). 

In structural testing, topological rules often demand monotonic inclusion (e.g., a child node can only be rejected if its parent is rejected). Because DOMT acts as a monotonic inflation operator ($\mathcal{R}_t^{base} \subseteq \mathcal{R}_t^{DOMT}$), it unconditionally preserves the structural sub-graph constraints satisfied by the baseline. Furthermore, under Positive Regression Dependency on a Subset (PRDS), monotonic threshold expansions generally preserve the conservative nature of FDR control. DOMT acts as a ``stochastic tunneling'' effect, allowing the algorithm to probabilistically bypass an erroneously depleted parent node and reignite discoveries deep within the network branches. Investigating how DOMT's non-negative stochastic expansion interacts with the topological ordering of graphical models (e.g., HMMs or DAG-based OMT) offers a robust mathematical foothold, positioning DOMT as a universal risk-control protocol across modern structured AI ecosystems.

\paragraph{Connections to Online Conformal Risk Control (CRC).}
While this work strictly operates within the canonical multiple testing framework, the underlying philosophy of dynamic risk-pricing via Weighted Regret naturally resonates with recent advancements in distribution-free uncertainty quantification. Specifically, in the presence of arbitrary distribution shifts, Adaptive Conformal Inference (ACI) \cite{gibbs2021adaptive, bastani2022practical} maintains dynamic coverage by treating the calibration parameter as a sequential learning problem. Currently, online CRC primarily minimizes the empirical coverage error without explicitly pricing the asymmetric cost of prediction set sizes. Bridging the DOMT exploration mechanism with ACI to develop an ``Asymmetric Regret-Optimal Conformal Predictor''---which safely inflates prediction sets during abrupt distributional shifts to rapidly recalibrate, while strictly bounding the long-term coverage deficit---represents a highly synergistic and mathematically profound future direction.

%%%%%%%%%%%%%%%%%%%%%%%%%%%%%%%%%%%%%%%%%%%%%%%%%%%%%%%%%%%%

% \newpage
% \input{checklist.tex}

% \bibliographystyle{plain}
% \bibliography{refs}
\end{document}